\documentclass{article}

\usepackage{times}
\usepackage{graphicx} 
\usepackage{subfigure} 

\usepackage{natbib}

\usepackage{algorithm}
\usepackage{algorithmic}

\usepackage{bm}
\usepackage{bbm}
\usepackage{amsmath, amsthm, amssymb}
\usepackage{mathtools}
\usepackage{multirow}
\usepackage{hhline}
\usepackage{booktabs}
\usepackage{multicol}

\usepackage[accepted, nohyperref]{icml2017}

\newtheorem{theorem}{Theorem}
\newtheorem{lemma}{Lemma}
\newtheorem{corollary}{Corollary}
\newtheorem{definition}{Definition}
\newtheorem{remark}{Remark}

\def\expct{{\mathbb{E}}}
\def\indic{{\mathbbm{1}}}

\def \nlabels{K}
\def \arity{M}
\def \nnodes{N}
\def \repr{\mathbf{r}}
\def \propor{q}

\DeclareMathOperator{\sign}{sign}

\def \nlabels{K}
\def \arity{M}
\def \nnodes{N}
\def \repr{\mathbf{r}}
\def \propor{q}

\icmltitlerunning{Simultaneous Learning of Trees and Representations for Extreme Classification and Density Estimation}

\begin{document} 

\twocolumn[
\icmltitle{Simultaneous Learning of Trees and Representations for Extreme Classification and Density Estimation}



\icmlsetsymbol{equal}{*}

\begin{icmlauthorlist}
\icmlauthor{Yacine Jernite}{to}
\icmlauthor{Anna Choromanska}{to}
\icmlauthor{David Sontag}{goo}
\end{icmlauthorlist}

\icmlaffiliation{to}{New York University, New York, New York, USA}
\icmlaffiliation{goo}{Massachussets Institute of Technology, Cambridge, Massachussets, USA}

\icmlcorrespondingauthor{Yacine Jernite}{jernite@cs.nyu.edu}
\icmlcorrespondingauthor{Anna Choromanska}{ac5455@nyu.edu}
\icmlcorrespondingauthor{David Sontag}{dsontag@mit.edu}


\vskip 0.3in
]



\printAffiliationsAndNotice{}  

\begin{abstract}
We consider multi-class classification where the predictor has a hierarchical structure that allows for a very large number of labels both at train and test time. The predictive power of such models can heavily depend on the structure of the tree, and although past work showed how to learn the tree structure, it expected that the feature vectors remained static. We provide a novel algorithm to simultaneously perform representation learning for the input data and learning of the hierarchical predictor. Our approach optimizes an objective function which favors balanced and easily-separable multi-way node partitions. We theoretically analyze this objective, showing that it gives rise to a boosting style property and a bound on classification error. We next show how to extend the algorithm to conditional density estimation. We empirically validate both variants of the algorithm on text classification and language modeling, respectively, and show that they compare favorably to common baselines in terms of accuracy and running time.
\end{abstract}

\section{Introduction}
Several machine learning settings are concerned with performing predictions in a very large discrete label space. From extreme multi-class classification to language modeling, one commonly used approach to this problem reduces it to a series of choices in a tree-structured model, where the leaves typically correspond to labels. While this allows for faster prediction, and is in many cases necessary to make the models tractable, the performance of the system can depend significantly on the structure of the tree used, e.g.~\cite{NIPS2008_3583}. 

Instead of relying on possibly costly heuristics~\cite{NIPS2008_3583}, extrinsic hierarchies~\cite{morin2005hierarchical} which can badly generalize across different data sets, or purely random trees, we provide an efficient data-dependent algorithm for tree construction and training. Inspired by the LOM tree algorithm~\cite{NIPS2015_5937} for binary trees, we present an objective function which favors high-quality node splits, i.e. balanced and easily separable. In contrast to previous work, our objective applies to trees of arbitrary width and leads to guarantees on model accuracy. Furthermore, we show how to successfully optimize it in the setting when the data representation needs to be learned simultaneously with the classification tree.

Finally, the multi-class classification problem is closely related to that of conditional density estimation~\cite{Ram:2011:DET:2020408.2020507, bishop:2006:PRML} since both need to consider all labels (at least implicitly) during learning and at prediction time. Both problems present similar difficulties when dealing with very large label spaces, and the techniques that we present in this work can be applied indiscriminately to either. Indeed, we show how to adapt our algorithm to efficiently solve the conditional density estimation problem of learning a language model which uses a tree structured objective.

This paper is organized as follows: Section \ref{sec:related} discusses related work, Section \ref{sec:bg} outlines the necessary background and defines the flat and tree-structured objectives for multi-class classification and density estimation, Section \ref{sec:algorithm} presents the objective and the optimization algorithm, Section \ref{sec:classlm} contains theoretical results, Section~\ref{sec:application} adapts the algorithm to the problem of language modeling, Section~\ref{sec:experiments} reports empirical results on the Flickr tag prediction dataset and Gutenberg text corpus, and finally Section~\ref{sec:conclusion} concludes the paper. Supplementary material contains additional material and proofs of theoretical statements of the paper. We also release the C++ implementation of our algorithm.

\section{Related Work}
\label{sec:related}

The multi-class classification problem has been addressed in the literature in a variety of ways. Some examples include i) clustering methods~\cite{BengioWG10,journals/informaticaSI/MadzarovGC09, weston13} (\cite{BengioWG10} was later improved in~\cite{DengSBL11}), ii) sparse output coding~\cite{conf/cvpr/ZhaoX13}, iii) variants of error correcting output codes~\cite{DBLP:journals/corr/abs-0902-1284}, iv) variants of iterative least-squares~\cite{DBLP:journals/corr/AgarwalKKSV13}, v) a method based on guess-averse loss functions~\cite{icml2014c2_beijbom14}, and vi) classification trees~\cite{BeygelzimerLR09,NIPS2015_5937,JL2016} (that includes the Conditional Probability Trees~\cite{BeygelzimerLLSS09}  when extended to the classification setting). 

The recently proposed LOM tree algorithm~\cite{NIPS2015_5937} differs significantly from other similar hierarchical approaches, like for example Filter Trees~\cite{BeygelzimerLR09} or random trees~\cite{Breiman:2001:RF:570181.570182}, in that it addresses the problem of learning good-quality binary node partitions. The method results in low-entropy trees and instead of using an inefficient enumerate-and-test approach, see e.g:~\cite{ig}, to find a good partition or expensive brute-force optimization~\cite{Manik}, it searches the space of all possible partitions with SGD~\cite{bottou-98x}. Another work~\cite{JL2016} uses a binary tree to map an example to a small subset of candidate labels and makes a final prediction via a more tractable one-against-all classifier, where this subset is identified with the proposed Recall Tree. A notable approach based on decision trees also include 
FastXML~\cite{Prabhu2014} (and its slower and less accurate at prediction predecessor~\cite{Manik}). It is based on optimizing the rank-sensitive loss function and shows an advantage over some other ranking and NLP-based techniques in the context of multi-label classification. Other related approaches include the SLEEC classifier~\cite{NIPS2015_5969} for extreme multi-label classification that learns embeddings which preserve pairwise distances between only the nearest label vectors and ranking approaches based on negative sampling~\cite{37180}. Another tree approach~\cite{export:255952} shows no computational speed up but leads to significant improvements in prediction accuracy. 

Conditional density estimation can also be challenging in settings where the label space is large. The underlying problem here consists in learning a probability distribution over a set of random variables given some context. For example, in the language modeling setting one can learn the probability of a word given the previous text, either by making a Markov assumption and approximating the left context by the last few words seen (n-grams e.g.~\cite{JelMer80,Katz1987}, feed-forward neural language models ~\cite{Mnih12afast, export:175560, conf/icassp/SchwenkG02}), or by attempting to learn a low-dimensional representation of the full history (RNNs~\cite{conf/interspeech/MikolovKBCK10, DBLP:journals/corr/MirowskiV15, DBLP:journals/corr/TaiSM15, DBLP:journals/corr/KumarISBEPOGS15}). Both the recurrent and feed-forward Neural Probabilistic Language Models (NPLM)~\cite{Bengio:2003:NPL:944919.944966} simultaneously learn a distributed representation for words and the probability function for word sequences, expressed in terms of these representations. The major drawback of these models is that they can be slow to train, as they grow linearly with the vocabulary size (anywhere between 10,000 and 1M words), which can make them difficult to apply~\cite{Mnih12afast}. A number of methods have been proposed to overcome this difficulty. Works such as LBL~\cite{mnih2007three} or Word2Vec~\cite{mikolov2013distributed} reduce the model to its barest bones, with only one hidden layer and no non-linearities. Another proposed approach has been to only compute the NPLM probabilities for a reduced vocabulary size, and use hybrid neural-$n$-gram model~\cite{Schwenk:2005:TNN:1220575.1220601} at prediction time. Other avenues to reduce the cost of computing gradients for large vocabularies include using different sampling techniques to approximate it~\cite{Bengio+Senecal-2003,journals/tnn/BengioS08,Mnih12afast}, replacing the likelihood objective by a contrastive one~\cite{gutmann2010noise} or spherical loss~\cite{DBLP:journals/corr/BrebissonV15}, relying on self-normalizing models~\cite{DBLP:conf/naacl/AndreasK15}, taking advantage of data sparsity ~\cite{DBLP:conf/nips/VincentBB15}, or using clustering-based methods~\cite{DBLP:journals/corr/GraveJCGJ16}. It should be noted however that most of these techniques (to the exception of \cite{DBLP:journals/corr/GraveJCGJ16}) do not provide any speed up at test time. 

Similarly to the classification case, there have also been a significant number of works that use tree structured models to accelerate computation of the likelihood and gradients~\cite{morin2005hierarchical, NIPS2008_3583, conf/www/DjuricWRGB15, mikolov2013distributed}. These use various heuristics to build a hierarchy, from using ontologies~\cite{morin2005hierarchical} to Huffman coding~\cite{mikolov2013distributed}. One algorithm which endeavors to learn a binary tree structure along with the representation is presented in ~\cite{NIPS2008_3583}. They iteratively learn word representations given a fixed tree structure, and use a criterion that trades off between making a balanced tree and clustering the words based on their current embedding. The application we present in the second part of our paper is most closely related to the latter work, and uses a similar embedding of the context. However, where their setting is limited to binary trees, we work with arbitrary width, and provide a tree building objective which is both less computationally costly and comes with theoretical guarantees.

\section{Background}
\label{sec:bg}

In this section, we define the classification and log-likelihood objectives we wish to maximize. Let $\mathcal{X}$ be an input space, and $\mathcal{V}$ a label space. Let $\mathcal{P}$ be a joint distribution over samples in $(\mathcal{X}, \mathcal{V})$, and let $f_{\Theta}: \mathcal{X} \rightarrow \mathbb{R}^{d_r}$ be a function mapping every input $x \in \mathcal{X}$ to a representation $\repr \in \mathbb{R}^{d_r}$, and parametrized by $\Theta$ (e.g. as a neural network).

 We consider two objectives. Let $g$ be a function that takes an input representation $\repr \in \mathbb{R}^{d_r}$, and predicts for it a label ${g(\repr) \in \mathcal{V}}$. The classification objective is defined as the expected proportion of correctly classified examples:
\vspace{-0.05in}
\begin{equation}
\mathcal{O^{\text{class}}}(\Theta, g) = \expct_{(x, y) \sim \mathcal{P}} \Big[ \indic [g \circ f_{\Theta}(x) = y] \Big]
\label{equ:class_obj}
\vspace{-0.05in}
\end{equation}
Now, let $p_\theta(\cdot | \repr)$ define a conditional probability distribution (parametrized by $\theta$) over $\mathcal{V}$ for any $\repr \in \mathbb{R}^{d_r}$. The density estimation task consists in maximizing the expected log-likelihood of samples from $(\mathcal{X}, \mathcal{V})$:
\vspace{-0.05in}
\begin{equation}
\mathcal{O^{\text{ll}}}(\Theta, \theta) = \expct_{(x, y) \sim \mathcal{P}} \Big[ \log p_\theta(y | f_{\Theta}(x)) \Big]
\label{equ:density_obj}
\end{equation}


\paragraph{Tree-Structured Classification and Density Estimation}
\label{sec:obj}

Let us now show how to express the objectives in Equations \ref{equ:class_obj} and \ref{equ:density_obj} when using tree-structured prediction functions (with fixed structure) as illustrated in Figure~\ref{fig:background_tree}.

\begin{figure}[htp!]
\vspace{-0.1in}
\centering
\includegraphics[width=0.45\textwidth]{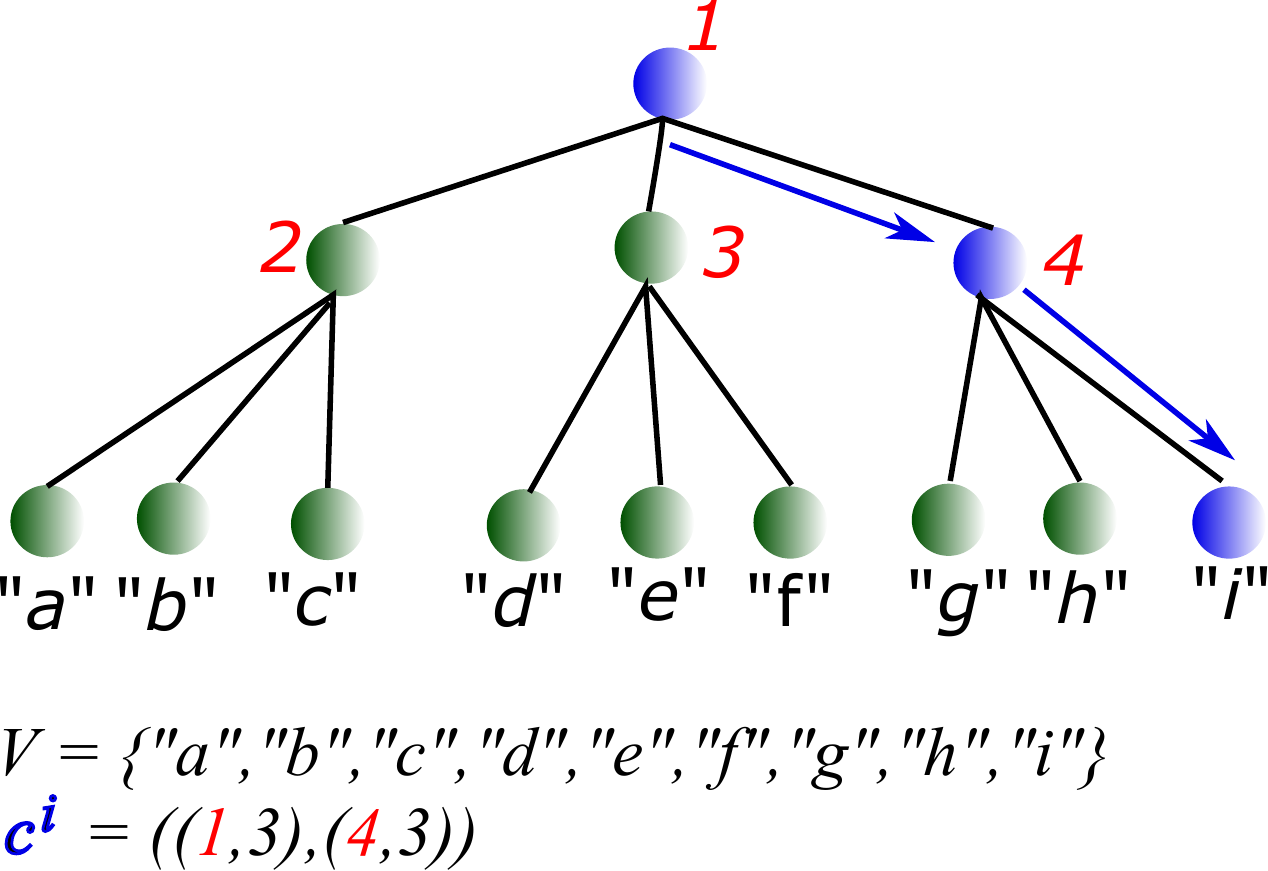}
\vspace{-0.1in}
\caption{Hierarchical predictor: in order to predict label ``$i$'', the system needs to choose the third child of node $1$, then the third child of node $4$.}
\label{fig:background_tree}
\end{figure}

Consider a tree $\mathcal{T}$ of depth $D$ and arity $\arity$ with $\nlabels = |\mathcal{V}|$ leaf nodes and $\nnodes$ internal nodes. Each leaf $l$ corresponds to a label, and can be identified with the path $\mathbf{c}^l$ from the root to the leaf. In the rest of the paper, we will use the following notations:
\vspace{-0.05in}
\begin{equation}
\mathbf{c}^l = ((c^l_{1, 1}, c^l_{1, 2}), \ldots, (c^l_{d, 1}, c^l_{d, 2}), \ldots, (c^l_{D, 1}, c^l_{D, 2})),
\label{eq:path}
\vspace{-0.05in}
\end{equation}
where $c^l_{d, 1} \in [1, \nnodes]$ correspond to the node index at depth $d$, and $c^l_{d, 2} \in [1, \arity]$ indicates which child of $c^l_{d, 1}$ is next in the path. In that case, our classification and density estimation problems are reduced to choosing the right child of a node or defining a probability distribution over children given $x \in \mathcal{X}$ respectively.

We then need to replace $g$ and $p_{\theta}$ with node decision functions $(g_n)_{n=1}^{\nnodes}$ and conditional probability distributions  $(p_{\theta_n})_{n=1}^{\nnodes}$ respectively. Given such a tree and representation function, our objective functions then become:
\vspace{-0.1in}
\begin{equation}
\hspace{-0.03in}\mathcal{O^{\text{class}}}(\Theta, g) = \expct_{(x, y) \sim \mathcal{P}} \Big[ \prod_{d=1}^D \indic [g_{c^l_{d, 1}} \circ f_{\Theta}(x) = c^l_{d, 2}] \Big]\hspace{-0.04in}
\label{eq:class_obj_tree}
\end{equation}
\vspace{-0.1in}
\begin{equation}
\hspace{-0.03in}\mathcal{O^{\text{ll}}}(\Theta, \theta) = \expct_{(x, y) \sim \mathcal{P}} \Big[\sum_{d=1}^D \log p_{\theta_{c^l_{d, 1}}}(c^l_{d, 2} | f_{\Theta}(x)) \Big]
\label{eq:density_obj_tree}
\end{equation}
\vspace{-0.15in}

The tree objectives defined in Equations \ref{eq:class_obj_tree} and \ref{eq:density_obj_tree} can be optimized in the space of parameters of the representation and node functions using standard gradient ascent methods. However, they also implicitly depend on the tree structure $\mathcal{T}$. In the rest of the paper, we provide a surrogate objective function which determines the structure of the tree and, as we show theoretically (Section \ref{sec:classlm}), maximizes the criterion in Equation \ref{eq:class_obj_tree} and, as we show empirically (Sections \ref{sec:application} and \ref{sec:experiments}), maximizes the criterion in Equation \ref{eq:density_obj_tree}.

\section{Learning Tree-Structured Objectives}
\label{sec:algorithm}

In this section, we introduce a per-node objective $J_n$ which leads to good quality trees when maximized, and provide an algorithm to optimize it.

\subsection{Objective function}

We define the node objective $J_n$ for node $n$ as:
 \vspace{-0.06in}
\begin{equation}
J_n = \frac{2}{\arity}\sum_{i=1}^{\nlabels}\propor^{(n)}_i \sum_{j=1}^{\arity}|p^{(n)}_j - p^{(n)}_{j|i}|,
\label{eq:objective}
 \vspace{-0.14in}
\end{equation}
where $\propor^{(n)}_i$ denotes the proportion of nodes reaching node $n$ that are of class $i$, $p^{(n)}_{j|i}$ is the probability that an example of class $i$ reaching $n$ will be sent to its $j^{\text{th}}$ child, and  $p^{(n)}_{j}$ is the probability that an example of any class reaching $n$ will be sent to its $j^{\text{th}}$ child. Note that we have:
 \vspace{-0.1in}
\begin{equation}
\forall j \in [1, \arity], \;\; p^{(n)}_j = \sum_{i=1}^{\nlabels} \propor^{(n)}_i p^{(n)}_{j|i}.
 \vspace{-0.1in}
\end{equation}
The objective in Equation~\ref{eq:objective} reduces to the LOM tree objective in the case of $\arity=2$.

At a high level, maximizing the objective encourages the conditional distribution for each class to be as different as possible from the global one; so the node decision function needs to be able to discriminate between examples of the different classes. The objective thus favors balanced and pure node splits. To wit, we call a split at node $n$ \textit{perfectly balanced} when the global distribution $p_{\cdot}^{(n)}$ is uniform, and \textit{perfectly pure} when each $p^{(n)}_{\cdot|i}$ takes value either $0$ or $1$, as all data points from the same class reaching node $n$ are sent to the same child.

In Section~\ref{sec:classlm} we discuss the theoretical properties of this objective in details. We show that maximizing it leads to perfectly balanced and perfectly pure splits. We also derive the boosting theorem that shows the number of internal nodes that the tree needs to have to reduce the classification error below any arbitrary threshold, under the assumption that the objective is ``weakly'' optimized in each node of the tree.

\begin{remark} 
In the rest of the paper, we use node functions $g_n$ which take as input a data representation $\repr \in \mathbb{R}^{d_r}$ and output a distribution over children of $n$ (for example using a soft-max function). When used in the classification setting, $g_n$ sends the data point to the child with the highest predicted probability. With this notation, and representation function $f_{\Theta}$, we can write:
\begin{equation}
p_j^{(n)} \coloneqq \mathbb{E}_{(x, y) \sim \mathcal{P}}[g_n \circ f_{\Theta}(x)] 
\label{eq:balprob}
\end{equation}
and
\begin{equation}
p_{j|i}^{(n)} \coloneqq \mathbb{E}_{(x, y) \sim \mathcal{P}}[g_n \circ f_{\Theta}(x) | y=i].
\label{eq:pureprob}
\end{equation}
An intuitive geometric interpretation of probabilities $p_j^{(n)}$ and $p_{j|i}^{(n)}$ can be found in the Supplementary material.
\end{remark}


\begin{algorithm}[t!]
\caption{Tree Learning Algorithm}
\label{algo:learn}
\begin{tabular}{l}
\textbf{Input} Input representation function: $f$ with parameters\\
\hspace{0.33in} $\Theta_f$. Node decisions functions $(g_n)_{n=1}^\nlabels$ with\\
\hspace{0.33in} parameters $(\Theta_n)_{n=1}^\nlabels$. Gradient step size $\epsilon$.\\
\textbf{Ouput} Learned $\arity$-ary tree, parameters $\Theta_f$ and $(\Theta_n)_{n=1}^\nlabels$.\\
\\
\textbf{procedure} \textsf{InitializeNodeStats} ()\\
\hspace{0.2in}\textbf{for} $n=1$ to $\nnodes$ \textbf{do}\\
\hspace{0.4in}\textbf{for} $i=1$ to $\nlabels$ \textbf{do}\\
\hspace{0.6in}$\text{SumProbas}_{n, i} \leftarrow \mathbf{0}$\\
\hspace{0.6in}$\text{Counts}_{n, i} \leftarrow 0$\\
\\
\textbf{procedure} \textsf{NodeCompute} ($\mathbf{w}$, $n$, $i$, target)\\
\hspace{0.2in}$\mathbf{p} \leftarrow g_{n}(\mathbf{w})$\\
\hspace{0.2in}$\text{SumProbas}_{n, i} \leftarrow \text{SumProbas}_{n, i} + \mathbf{p}$\\
\hspace{0.2in}$\text{Counts}_{n, i} \leftarrow \text{Counts}_{n, i} + 1$\\
\hspace{0.2in}// {\sl{Gradient step in the node parameters}}\\
\hspace{0.2in}$\Theta_{n} \leftarrow \Theta_{n} + \epsilon \frac{\partial p_{\text{target}}}{\partial \Theta_{n}}$ \label{algo:line:grad_1}\\
\hspace{0.2in} \textbf{return} $\frac{\partial p_{\text{target}}}{\partial \mathbf{w}}$ \label{algo:line:grad_2}\\
\\
\textsf{InitializeNodeStats} ()\\
\textbf{for} Each batch $b$ \textbf{do}\\
\hspace{0.2in}// {\sl{\textsf{AssignLabels} () re-builds the tree based on the}}\\
\hspace{0.2in}// {\sl{current statistics}}\\
\hspace{0.2in} \textsf{AssignLabels} ($\{1, \ldots, \nlabels\}$, root)\\
\hspace{0.2in} \textbf{for} each example $(\mathbf{x}, i)$ in $b$ \textbf{do}\\
\hspace{0.4in} Compute input representation $\mathbf{w} = f(\mathbf{x})$\\
\hspace{0.4in} $\Delta \mathbf{w} \leftarrow \mathbf{0}$\\
\hspace{0.4in} \textbf{for} $d=1$ to $D$ \textbf{do}\\
\hspace{0.6in} Set node id and target: $(n, j) \leftarrow c^i_{d}$\\
\hspace{0.6in} $\Delta \mathbf{w} \leftarrow \Delta \mathbf{w}$ + \textsf{NodeCompute} ($\mathbf{w}$, n, i, j) \label{algo:line:grad_node}\\
\\
\hspace{0.4in} // {\sl{Gradient step in the parameters of $f$}}\\
\hspace{0.4in} $\Theta_f \leftarrow \Theta_f + \epsilon \frac{\partial f}{\partial \Theta_f} \Delta \mathbf{w}$ \label{algo:line:grad_rep}
\end{tabular}
\end{algorithm}

\subsection{Algorithm}

In this section we present an algorithm for simultaneously building the classification tree and  learning the data representation. We aim at maximizing the accuracy of the tree as defined in Equation~\ref{eq:class_obj_tree} by maximizing the objective $J_n$ of Equation~\ref{eq:objective} at each node of the tree (the boosting theorem that will be presented in Section~\ref{sec:classlm} shows the connection between the two).

\begin{algorithm}[t!]
\caption{Label Assignment Algorithm}
\label{algo:assign}
\begin{tabular}{l}
\textbf{Input} labels currently reaching the node\\
\hspace{0.33in} node ID $n$\\
\textbf{Ouput}  Lists of labels now assigned to the node's children\\
\\
\textbf{procedure} \textsf{CheckFull} (full, assigned, count, $j$)\\
\hspace{0.2in} \textbf{if} $|\text{assigned}_{j}|  \equiv 2 \mod (M-1)$ \textbf{then}\\
\hspace{0.4in} count $\leftarrow \text{count} - (M-1)$\\
\hspace{0.2in} \textbf{if} $\text{count} = 0$ \textbf{then}\\
\hspace{0.4in} $\text{full} \leftarrow \text{full} \cup \{j\} $\\
\hspace{0.2in} \textbf{if} $\text{count} = 1$ \textbf{then}\\
\hspace{0.4in} count $\leftarrow 0$\\
\hspace{0.4in} \textbf{for} $j'$ s.t. $|\text{assigned}_{j'}|  \equiv 1 \mod (M-1)$ \textbf{do}\\
\hspace{0.6in} $\text{full} \leftarrow \text{full} \cup \{j'\} $\\
\\
\textbf{procedure} \textsf{AssignLabels} (labels, $n$)\\
\hspace{0.2in} // {\sl{first, compute $p_j^{(n)}$ and $p_{j|i}^{(n)}$.}}\\
\hspace{0.2in} $\mathbf{p}^{avg}_0 \leftarrow \mathbf{0}$\\
\hspace{0.2in} $\text{count} \leftarrow 0$\\
\hspace{0.2in} \textbf{for} $i$ in labels \textbf{do}\\
\hspace{0.4in} $\mathbf{p}^{avg}_0 \leftarrow \mathbf{p}^{avg}_0 + \text{SumProbas}_{n, i}$\\
\hspace{0.4in} $\text{count} \leftarrow \text{count} + \text{Counts}_{n, i}$\\
\hspace{0.4in} $\mathbf{p}^{avg}_i \leftarrow \text{SumProbas}_{n, i} / \text{Counts}_{n, i}$\\
\hspace{0.2in} $\mathbf{p}^{avg}_0 \leftarrow \mathbf{p}^{avg}_0 / \text{count}$\\
\hspace{0.2in} // {\sl{then, assign each label to a child of $n$}}\\
\hspace{0.2in} unassigned $\leftarrow$ labels\\
\hspace{0.2in} full $\leftarrow \emptyset$\\
\hspace{0.2in} count $\leftarrow (|\text{unassigned}| - (M - 1))$\\
\hspace{0.2in} \textbf{for} $j=1$ to $\arity$ \textbf{do}\\
\hspace{0.4in} $\text{assigned}_j \leftarrow \emptyset$\\
\hspace{0.2in} \textbf{while} $\text{unassigned} \neq \emptyset$ \textbf{do}\\
\hspace{0.4in} \Big{/}\!\!\Big{/}{\sl{$\frac{\partial J_n}{ \partial p^{(n)}_{j|i}}$ is given in Equation \ref{eq:gradients}}}\\
\hspace{0.4in} $(i^*, j^*) \leftarrow \operatorname*{argmax}\limits_{i \in \text{unassigned}, j \not \in \text{full}}\left(\frac{\partial J_n}{ \partial p^{(n)}_{j|i}}\right)$ \label{algo:line:grad_sort}\\
\hspace{0.4in} \textbf{if} $n=\text{root}$ \textbf{then} \\
\hspace{0.6in} $\mathbf{c}^{i^*} \leftarrow (n, j^*)$\\
\hspace{0.4in} \textbf{else} \\
\hspace{0.6in} $\mathbf{c}^{i^*} \leftarrow (\mathbf{c}^{i^*}, (n, j^*))$\\
\hspace{0.4in} $\text{assigned}_{j^*} \leftarrow \text{assigned}_{j^*} \cup \{ i^* \}$\\
\hspace{0.4in} $\text{unassigned} \leftarrow \text{unassigned} \setminus \{ i^* \}$\\
\hspace{0.4in} \textsf{CheckFull} (full, assigned, count, $j^*$)\\
\hspace{0.2in} \textbf{for} $j=1$ to $\arity$ \textbf{do}\\
\hspace{0.4in} \textsf{AssignLabels} ($\text{assigned}_j$, $\text{child}_{n, j}$, $d + 1$)\\
\hspace{0.2in} \textbf{return} assigned
\end{tabular}
\end{algorithm}
\setlength{\textfloatsep}{15pt}

Let us now show how we can efficiently optimize $J_n$. The gradient of $J_n$ with respect to the conditional probability distributions is (see proof of Lemma \ref{lem:subgrad} in the Supplement):
\vspace{-0.05in}
\begin{equation}
\frac{\partial J_n}{ \partial p^{(n)}_{j|i}} = \frac{2}{\arity} \propor^{(n)}_i (1 - \propor^{(n)}_i) \sign(p^{(n)}_{j|i} - p^{(n)}_j).
\label{eq:gradients}
 \vspace{-0.15in}
\end{equation}

Then, according to Equation \ref{eq:gradients}, increasing the likelihood of sending label $i$ to any child $j$ of $n$ such that $p^{(n)}_{j|i} > p^{(n)}_j$ increases the objective $J_n$. Note that we only need to consider the labels $i$ for which $\propor^{(n)}_i>0$, that is, labels $i$ which reach node $n$ in the current tree.

We also want to make sure that we have a well-formed $\arity$-ary tree at each step, which means that the number of labels assigned to any node is always congruent to $1$ modulo $(M - 1)$. Algorithm \ref{algo:assign} provides such an assignment by greedily choosing the label-child pair $(i, j)$ such that $j$ still has room for labels with the highest value of $\frac{\partial J_n}{ \partial p^{(n)}_{j|i}}$.

The global procedure, described in Algorithm \ref{algo:learn}, is then the following.
\begin{itemize}
\vspace{-0.15in}
\item At the start of each batch, re-assign targets for each node prediction function, starting from the root and going down the tree. At each node, each label is more likely to be re-assigned to the child it has had most affinity with in the past (Algorithm \ref{algo:assign}). This can be seen as a form of hierarchical on-line clustering.
\vspace{-0.05in}
\item Every example now has a unique path depending on its label. For each sample, we then take a gradient step at each node along the assigned path (see Algorithm \ref{algo:learn}).
\end{itemize}

\begin{lemma}
Algorithm~\ref{algo:assign} finds the assignment of nodes to children for a fixed depth tree which most increases $J_n$ under well-formedness constraints.
\label{lem:subgrad}
\end{lemma}
\begin{remark}
An interesting feature of the algorithm, is that since the representation of examples from different classes are learned together, there is intuitively less of a risk of getting stuck in a specific tree configuration. More specifically, if two similar classes are initially assigned to different children of a node, the algorithm is less likely to keep this initial decision since the representations for examples of both classes will be pulled together in other nodes.
\end{remark}
Next, we provide a theoretical analysis of the objective introduced in Equation~\ref{eq:objective}. Proofs are deferred to the Supplementary material.

\section{Theoretical Results}
\label{sec:classlm}

In this section, we first analyze theoretical properties of the objective $J_n$ as regards node quality, then prove a boosting statement for the global tree accuracy.

\subsection{Properties of the objective function}

We start by showing that maximizing $J_n$ in every node of the tree leads to high-quality nodes, i.e. perfectly balanced and perfectly pure node splits. Let us first introduce some formal definitions.
\begin{definition}[Balancedness factor]
The split in node $n$ of the tree is $\beta^{(n)}$-balanced if
\vspace{-0.05in}
\[\beta^{(n)} \leq \min_{j = \{1,2,\dots,M\}}p_j^{(n)},
\]
\vspace{-0.05in}
where $\beta^{(n)} \in (0,\frac{1}{\arity}]$ is a balancedness factor.
\label{def:bal}
\end{definition}
A split is perfectly balanced if and only if $\beta^{(n)} = \frac{1}{\arity}$.
\begin{definition}[Purity factor]
The split in node $n$ of the tree is $\alpha^{(n)}$-pure if
\vspace{-0.05in}
\[\frac{1}{\arity}\sum_{j=1}^\arity\sum_{i=1}^\nlabels \propor_i^{(n)}\min\left(p_{j|i}^{(n)},1-p_{j|i}^{(n)}\right) \leq \alpha^{(n)},
\]
\vspace{-0.05in}
where $\alpha^{(n)} \in [0,\frac{1}{\arity})$ is a purity factor.
\label{def:purity}
\end{definition}
A split is perfectly pure if and only if $\alpha^{(n)} = 0$.

The following lemmas characterize the range of the objective $J_n$ and link it to the notions of balancedness and purity of the split.

\begin{lemma}
The objective function $J_n$ lies in the interval $\left[0,\frac{4}{\arity}\left(1 - \frac{1}{\arity}\right)\right]$.
\label{lem:avgscorerange}
\end{lemma}

Let $J^{*}$ denotes the highest possible value of $J_n$, i.e. $J^{*} = \frac{4}{\arity}\left(1 - \frac{1}{\arity}\right)$.

\begin{lemma}
The objective function $J_n$  admits the highest value, i.e. $J_n = J^{*}$, if and only if the split  in node $n$ is perfectly balanced, i.e. $\beta^{(n)} = \frac{1}{M}$, and perfectly pure, i.e. $\alpha^{(n)} = 0$.
\label{lem:avgscoremaxvalue}
\end{lemma}
We next show Lemmas~\ref{lem:isol_bal} and~\ref{lem:isol_pur} which analyze balancedness and purity of a node split in isolation, i.e. we analyze resp. balancedness and purity of a node split when resp. purity and balancedness is fixed and perfect. We show that in such isolated setting increasing $J_n$ leads to a more balanced and more pure split. 

\begin{lemma}
If a split in node $n$ is perfectly pure, then
\vspace{-0.1in}
\[\beta^{(n)} \in \left[\frac{1}{\arity} - \frac{\sqrt{M(J^{*} - J_n)}}{2},\frac{1}{\arity}\right].
\]
\label{lem:isol_bal}
\vspace{-0.15in}
\end{lemma}
\begin{lemma}
If a split in node $n$ is perfectly balanced, then $\alpha^{(n)} \leq (J^{*} - J_n)/2$.
\label{lem:isol_pur}
\end{lemma}

Next we provide a bound on the classification error for the tree. In particular, we show that if the objective is ``weakly'' optimized in each node of the tree, where this weak advantage is captured in a form of the \textit{Weak Hypothesis Assumption}, then our algorithm will amplify this weak advantage to build a tree achieving any desired level of accuracy.

\subsection{Error bound}

Denote $y(x)$ to be a fixed target function with domain $\mathcal{X}$, which assigns the data point $x$ to its label, and let $\mathcal{P}$ be a fixed target distribution over $\mathcal{X}$. Together $y$ and $\mathcal{P}$ induce a distribution on labeled pairs $(x,y(x))$. Let $t(x)$ be the label assigned to data point $x$ by the tree. We denote as $\epsilon(\mathcal{T})$ the error of tree $\mathcal{T}$, i.e. ${\epsilon(\mathcal{T}) \coloneqq \expct_{x \sim \mathcal{P}} \Big[ \sum_{i=1}^{\nlabels} \indic [t(x) = i, y(x) \neq i ] \Big]}$ (${1 - \epsilon(\mathcal{T})}$ refers to the accuracy as given by Equation~\ref{eq:class_obj_tree}). Then the following theorem holds

\begin{theorem}
The Weak Hypothesis Assumption says that for any distribution $\mathcal{P}$ over the data, at each node $n$ of the tree $\mathcal{T}$ there exists a partition such that $J_n \geq \gamma$, where $\gamma \in \left[\frac{\arity}{2}\operatorname*{min}\limits_{j=1,2,\dots,\arity}p_j, 1 - \frac{\arity}{2}\operatorname*{min}\limits_{j=1,2,\dots,\arity}p_j\right]$.

Under the Weak Hypothesis Assumption, for any $\kappa \in [0,1]$, to
obtain $\epsilon(\mathcal{T}) \leq \kappa$ it suffices to have a tree with
\vspace{-0.05in}
\[N \geq \left(\frac{1}{\kappa}\right)^{\frac{16[\arity(1-2\gamma) + 2\gamma](\arity-1)}{\log_2 e \arity^2\gamma^2}\ln
  \nlabels} \:\:\:\:\:\:\:\:\:\:\text{internal nodes}.
\] 
\label{thm:errorboundgen}
\vspace{-0.2in}
\end{theorem}
The above theorem shows the number of splits that suffice to reduce the multi-class classification error of the tree below an arbitrary threshold $\kappa$. As shown in the proof of the above theorem, the \textit{Weak Hypothesis Assumption} implies that all $p_j$s satisfy: $p_j \in [\frac{2\gamma}{\arity},\frac{\arity(1-2\gamma) + 2\gamma}{\arity}]$. Below we show a tighter version of this bound when assuming that each node induces balanced split.

\begin{corollary}
The Weak Hypothesis Assumption says that for any distribution $\mathcal{P}$ over the data, at each node $n$ of the tree $\mathcal{T}$ there exists a partition such that $J_n \geq \gamma$, where $\gamma \in \mathbb{R}^{+}$.

Under the Weak Hypothesis Assumption and when all nodes make perfectly balanced splits, for any $\kappa \in [0,1]$, to
obtain $\epsilon(\mathcal{T}) \leq \kappa$ it suffices to have a tree with
\vspace{-0.05in}
\[N \geq \left(\frac{1}{\kappa}\right)^{\frac{16(\arity-1)}{\log_2 e \arity^2\gamma^2}\ln
  \nlabels} \:\:\:\:\:\:\:\:\:\:\text{internal nodes}.
\] 
\label{corr:errorboundgensquare}
\end{corollary}
\section{Extension to Density Estimation}
\label{sec:application}

We now show how to adapt the algorithm presented in Section~\ref{sec:algorithm} for conditional density estimation, using the example of language modeling. 

\paragraph{Hierarchical Log Bi-Linear Language Model (HLBL)} We take the same approach to language modeling as \cite{NIPS2008_3583}. First, using the chain rule and an order $T$ Markov assumption we model the probability of a sentence $\mathbf{w} = (w_1, w_2, \ldots, w_n)$ as:
\vspace{-0.13in}
\begin{equation*}
p(w_1, w_2, \ldots, w_n) =  \prod_{t=1}^n p(w_t|w_{t-T, \ldots, t-1})
\vspace{-0.07in}
\end{equation*}
Similarly to their work, we also use a low dimensional representation of the context $(w_{t-T, \ldots, t-1})$. In this setting, each word $w$ in the vocabulary $\mathcal{V}$ has an embedding $U_w \in \mathbb{R}^{d_r}$. A given context $x = (w_{t-T}, \ldots, w_{t-1})$ corresponding to position $t$ is then represented by a context embedding vector $r_x$ such that
\vspace{-0.09in}
\[r_x = \sum_{k=1}^T R_k U_{w_{t-k}},
\vspace{-0.09in}\]
where $U \in \mathbb{R}^{|\mathcal{V}| \times {d_r}}$ is the embedding matrix, and $R_k \in \mathbb{R}^{{d_r} \times {d_r}}$ is the transition matrix associated with the $k^{\text{th}}$ context word.

The most straight-forward way to define a probability function is then to define the distribution over the next word given the context representation as a soft-max, as done in \cite{mnih2007three}. That is:
\vspace{-0.07in}
\begin{eqnarray*}
p(w_t = i| x) &=& \sigma_i(r_x^{\top} U + \mathbf{b}) \\
&=& \frac{\exp(r_x^{\top} U_{i} + b_{i})}{\sum_{w \in \mathcal{V}} \exp(r_x^{\top} U_{w} + b_{w}) } ,
\end{eqnarray*}
where $b_w$ is the bias for word $w$. However, the complexity of computing this probability distribution in this setting is $O(|V| \times d_r)$, which can be prohibitive for large corpora and vocabularies.

Instead, \cite{NIPS2008_3583} takes a hierarchical approach to the problem. They construct a binary tree, where each word $w \in \mathcal{V}$ corresponds to some leaf of the tree, and can thus be identified with the path from the root to the corresponding leaf by making a sequence of choices of going left versus right. This corresponds to the tree-structured log-likelihood objective presented in Equation \ref{eq:density_obj_tree} for the case where $M=2$, and $f_{\Theta}(x)=r_x$. Thus, if $\mathbf{c}^i$ is the path to word $i$ as defined in Expression~\ref{eq:path}, then:
\vspace{-0.1in}
\begin{equation}
\log p(w_t = i| x) = \sum_{d=1}^D\log \sigma_{c_{d, 2}^i}((r_x^{\top} U^{c_{d, 1}^i} + \mathbf{b}^{c_{d, 1}^i})
\label{eq:tree_proba}
\vspace{-0.05in}
\end{equation}
In this binary case, $\sigma$ is the sigmoid function, and for all non-leaf nodes $n \in \{1,2,\dots,\nnodes\}$, we have $U^n \in \mathbb{R}^{d_r}$ and $\mathbf{b}^n \in \mathbb{R}^{d_r}$. The cost of computing the likelihood of word $w$ is then reduced to $O(\log(|\mathcal{V}|) \times d_r)$. In their work, the authors start the training procedure by using a random tree, then alternate parameter learning with using a clustering-based heuristic to rebuild their hierarchy. We expand upon their method by providing an algorithm which allows for using hierarchies of arbitrary width, and jointly learns the tree structure and the model parameters.

\paragraph{Using our Algorithm} We may use Algorithm \ref{algo:learn} as is to learn a good tree structure for classification: that is, a model that often predicts $w_t$ to be the most likely word after seeing the context $(w_{t-T}, \ldots, w_{t-1})$. However, while this could certainly learn interesting representations and tree structure, there is no guarantee that such a model would achieve a good average log-likelihood. Intuitively, there are often several valid possibilities for a word given its immediate left context, which a classification objective does not necessarily take into account. Yet another option would be to learn a tree structure that maximizes the classification objective, then fine-tune the model parameters using the log-likelihood objective. We tried this method, but initial tests of this approach did not do much better than the use of random trees. Instead, we present here a small modification of Algorithm \ref{algo:learn} which is equivalent to log-likelihood training when restricted to the fixed tree setting, and can be shown to increase the value of the node objectives $J_n$: by replacing the gradients with respect to $p_{target}$ by those with respect to $\log p_{target}$. Then, for a given tree structure, the algorithm takes a gradient step with respect to the log-likelihood of the samples:
 \vspace{-0.03in}
\begin{equation}
\hspace{-0.02in}\frac{\partial J_n}{ \partial \log p^{(n)}_{j|i}} = \frac{2}{\arity} \propor^{(n)}_i (1 \!-\! \propor^{(n)}_i) \sign(p^{(n)}_{j|i} \!-\! p^{(n)}_j) p^{(n)}_{j|i}.\hspace{-0.05in}
\label{eq:log_gradients}
 \vspace{-0.04in}
\end{equation}
Lemma \ref{lem:subgrad} extends to the new version of the algorithm.

\setcounter{figure}{2}
\begin{figure*}[htp!]
\centering
\includegraphics[width=1\textwidth]{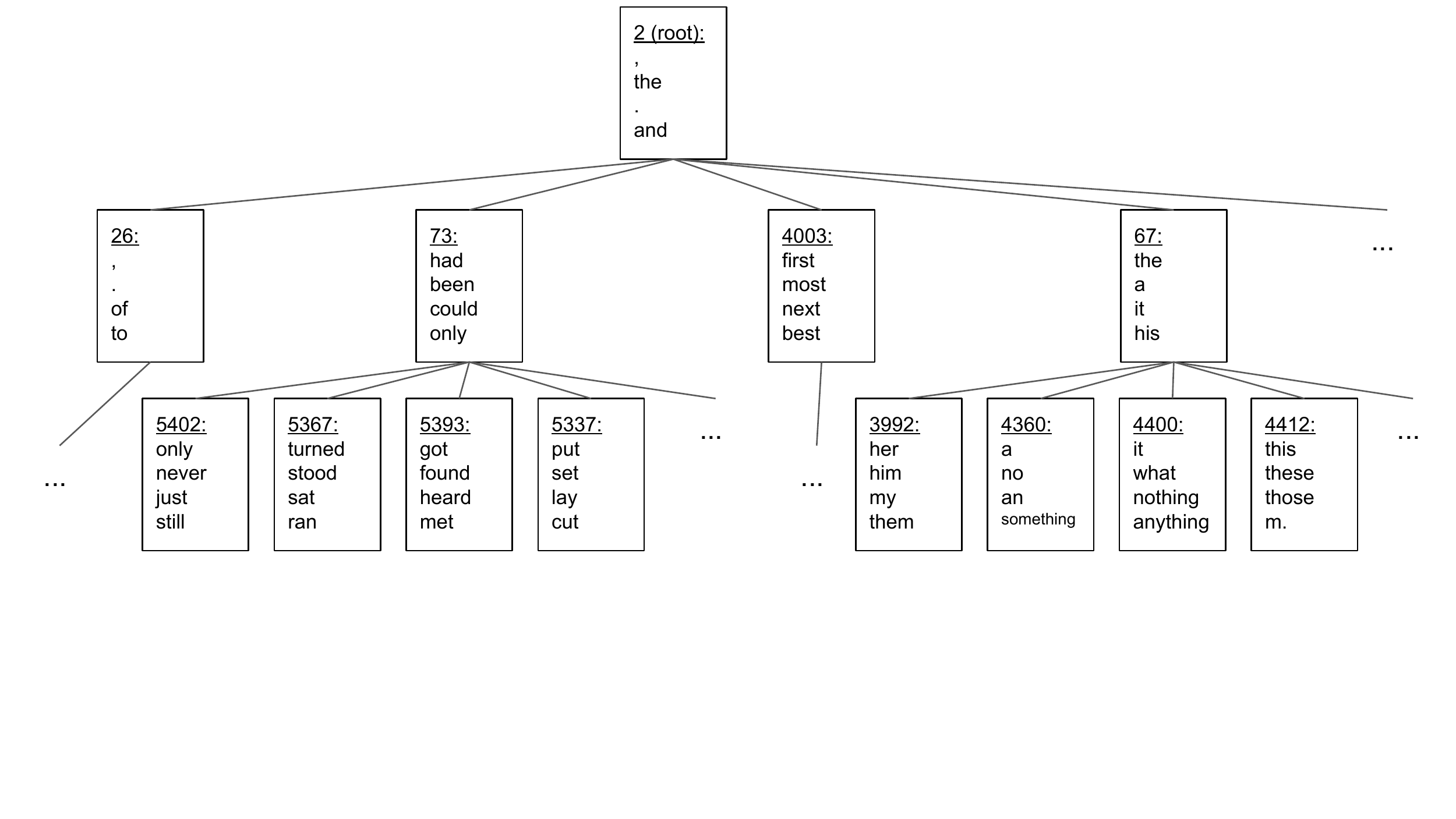}
\vspace{-1.5in}
\caption{Tree learned from the Gutenberg corpus, showing the four most common words assigned to each node.}
\label{fig:gutem_tree}
\end{figure*}

\section{Experiments}
\label{sec:experiments}

We ran experiments to evaluate both the classification and density estimation version of our algorithm. For classification, we used the YFCC100M dataset \cite{DBLP:journals/cacm/ThomeeSFENPBL16}, which consists of a set of a hundred million Flickr pictures along with captions and tag sets split into 91M training, 930K validation and 543K test examples. We focus here on the problem of predicting a picture's tags given its caption. For density estimation, we learned a log-bilinear language model on the Gutenberg novels corpus, and compared the perplexity to that obtained with other flat and hierarchical losses. Experimental settings are described in greater detail in the Supplementary material.

\subsection{Classification}

 We follow the setting of \cite{DBLP:journals/corr/JoulinGBM16} for the YFCC100M tag prediction task: we only keep the tags which appear at least a hundred times, which leaves us with a label space of size 312K. We compare our results to those obtained with the FastText software \cite{DBLP:journals/corr/JoulinGBM16}, which uses a binary hierarchical softmax objective based on Huffman coding (Huffman trees are designed to minimize the expected depth of their leaves weighed by frequencies and have been shown to work well with word embedding systems \cite{mikolov2013distributed}), and to the Tagspace system \cite{DBLP:conf/emnlp/WestonCA14}, which uses a sampling-based margin loss (this allows for training in tractable time, but does not help at test time, hence the long times reported).
 We also extend the FastText software to use Huffman trees of arbitrary width. All models use a bag-of-word embedding representation of the caption text; the parameters of the input representation function $f_{\Theta}$ which we learn are the word embeddings $U_{w} \in \mathbb{R}^d$ (as in Section \ref{sec:application}) and a caption representation is obtained by summing the embeddings of its words. We experimented with embeddings of dimension $d=50$ and $d=200$. We predict one tag for each caption, and report the precision as well as the training and test times in Table \ref{tab:class-res}.

\begin{table}[t!]
\setlength{\tabcolsep}{4.5pt}
\vspace{-0.025in}
\centering
\begin{tabular}{c l c c c c}
\toprule
$d$ &\multicolumn{1}{c}{Model} & Arity & P@1 & Train & Test \\ \midrule
\multirow{6}{*}{50} & TagSpace$^1$   & -  & 30.1  & 3h8   & 6h\\ \cmidrule{2-6}
 & FastText$^2$  & 2    & 27.2  & \textbf{8m} & \textbf{1m}   \\ \cmidrule{2-6}
 & \multirow{2}{*}{$M$-ary Huffman Tree} & 5 & 28.3   & \textbf{8m}  & \textbf{1m}  \\
 &    & 20   & 29.9   & 10m   & 3m\\ \cmidrule{2-6}
 & \multirow{2}{*}{Learned Tree} & 5 & 31.6  & 18m   & \textbf{1m}   \\ 
 &         & 20   & \textbf{32.1}   & 30m   & 3m   \\ \midrule
\multirow{6}{*}{200} & TagSpace$^1$  & 35.6 & 5h32  & 15h  \\ \cmidrule{2-6}
 & FastText$^2$ & 2 & 35.2  & \textbf{12m} & \textbf{1m}   \\ \cmidrule{2-6}
 & \multirow{2}{*}{$M$-ary Huffman Tree} & 5 & 35.8 & 13m   & 2m   \\
  & & 20   & 36.4    & 18m   & 3m   \\ \cmidrule{2-6}
 &\multirow{2}{*}{Learned Tree} & 5  & 36.1  & 35m   & 3m   \\
  &  & 20    & \textbf{36.6}  & 45m   & 8m   \\ \bottomrule
\end{tabular}
\caption{Classification performance on the YFCC100M dataset. $^1$\cite{DBLP:conf/emnlp/WestonCA14}. $^2$\cite{DBLP:journals/corr/JoulinGBM16}. $M$-ary Huffman Tree modifies FastText by adding an $M$-ary hierarchical softmax objective.}
\label{tab:class-res}
\end{table}

Our implementation is based on the FastText open source version\footnote{https://github.com/facebookresearch/fastText}, to which we added $M$-ary Huffman and learned tree objectives. Table \ref{tab:class-res} reports the best accuracy we obtained with a hyper-parameter search using this version on our system so as to provide the most meaningful comparison, even though the accuracy is less than that reported in \cite{DBLP:journals/corr/JoulinGBM16}.

We gain a few different insights from Table \ref{tab:class-res}. First, although wider trees are theoretically slower (remember that the theoretical complexity is $O(M\log_{M}(N))$ for an $M$-ary tree with $N$ labels), they run incomparable time in practice and always perform better. Using our algorithm to learn the structure of the tree also always leads to more accurate models, with a gain of up to 3.3 precision points in the smaller 5-ary setting. Further, both the importance of having wider trees and learning the structure seems to be less when the node prediction functions become more expressive. At a high level, one could imagine that in that setting, the model can learn to use different dimensions of the input representation for different nodes, which would minimize the negative impact of having to learn a representation which is suited to more nodes.

Another thing to notice is that since prediction time only depends on the expected depth of a label, our models which learned balanced trees are nearly as fast as Huffman coding which is optimal in that respect (except for the dimension 200, 20-ary tree, but the tree structure had not stabilized yet in that setting). Given all of the above remarks, our algorithm especially shines in settings where computational complexity and prediction time are highly constrained at test time, such as mobile devices or embedded systems.

\subsection{Density Estimation}

We also ran language modeling experiments on the Gutenberg novel corpus\footnote{http://www.gutenberg.org/}, which has about 50M tokens and a vocabulary of 250,000 words.

\setcounter{figure}{1}

\begin{figure}[t!]
\vspace{-0.15in}
\centering
\includegraphics[width=0.49\textwidth]{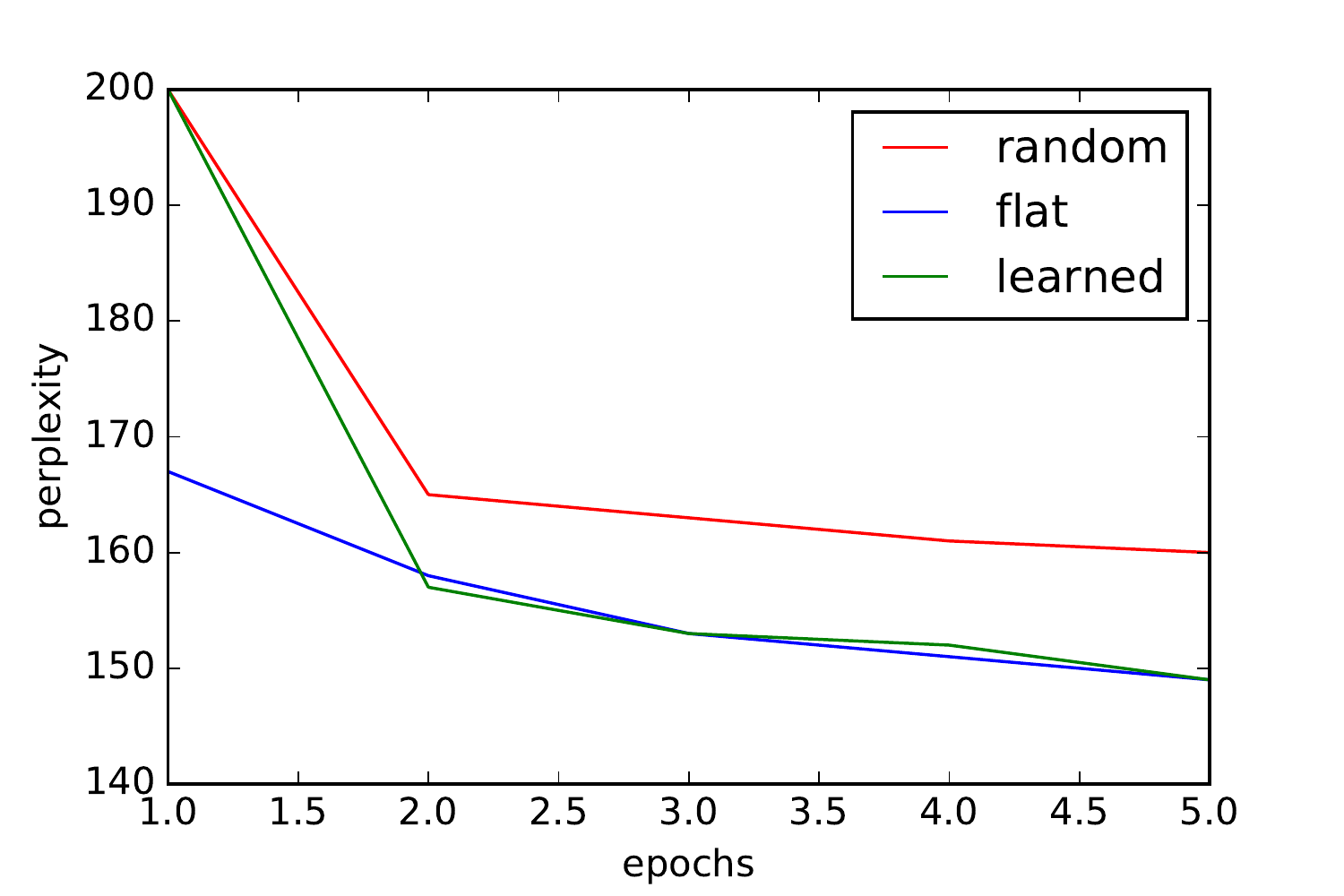}
\vspace{-0.25in}
\caption{Test perplexity per epoch.}
\label{fig:learn_curves}
\vspace{-0.07in}
\end{figure}

One notable difference from the previous task is that the language modeling setting can drastically benefit from the use of GPU computing, which can make using a flat softmax tractable (if not fast). While our algorithm requires more flexibility and thus does not benefit as much from the use of GPUs, a small modification of Algorithm \ref{algo:assign} (described in the Supplementary material) allows it to run under a maximum depth constraint and remain competitive. The results presented in this section are obtained using this modified version, which learns 65-ary trees of depth 3.

Table \ref{tab:perplexity} presents perplexity results for different loss functions, along with the time spent on computing and learning the objective (softmax parameters for the flat version, hierarchical softmax node parameters for the fixed tree, and hierarchical softmax structure and parameters for our algorithm). The learned tree model is nearly three and seven times as fast at train and test time respectively as the flat objective without losing any points of perplexity.

\begin{table}[tp!]
\setlength{\tabcolsep}{5pt}
\centering
\vspace{0.1in}
\begin{tabular}{l c c c}
\toprule
Model      & perp. & train ms/batch & test ms/batch \\ \midrule
Clustering Tree & 212  & \textbf{2.0} & \textbf{1.0} \\ \midrule
Random Tree     & 160  & \textbf{1.9} & \textbf{0.9} \\ \midrule
Flat soft-max   & 149 & \textit{12.5} & \textit{6.9} \\ \midrule
Learned Tree    & \textbf{148} & \textit{4.5} & \textbf{0.9} \\ \bottomrule
\end{tabular}
\caption{Comparison of a flat soft-max to a 65-ary hierarchical soft-max (learned, random and heuristic-based tree).}
\label{tab:perplexity}
\end{table}

Huffman coding does not apply to trees where all of the leaves are at the same depth. Instead, we use the following heuristic as a baseline, inspired by \cite{NIPS2008_3583}: we learn word embeddings using FastText, perform a hierarchical clustering of the vocabulary based on these, then use the resulting tree to learn a new language model. We call this approach ``Clustering Tree''. However, for all hyper-parameter settings, this tree structure did worse than a random one. We conjecture that its poor performance is because such a tree structure means that the deepest node decisions can be quite difficult.

Figure \ref{fig:learn_curves} shows the evolution of the test perplexity for a few epochs. It appears that most of the relevant tree structure can be learned in one epoch: from the second epoch on, the learned hierarchical soft-max performs similarly to the flat one. Figure \ref{fig:gutem_tree} shows a part of the tree learned on the Gutenberg dataset, which appears to make semantic and syntactic sense.

\vspace{-0.02in}
\section{Conclusion}
\label{sec:conclusion}

In this paper, we introduced a provably accurate algorithm for jointly learning tree structure and data representation for hierarchical prediction. We applied it to a multi-class classification and a density estimation problem, and showed our models' ability to achieve favorable accuracy in competitive times in both settings.

\newpage

\bibliography{LearnFeatMultiClass}
\bibliographystyle{icml2017}

\newpage
\normalsize
\onecolumn

\vbox{\hsize\textwidth
\linewidth\hsize \vskip 0.1in \toptitlebar \centering
{\Large\bf Simultaneous Learning of Trees and Representations for Extreme Classification with Application to Language Modeling\\ (Supplementary material) \par}
\bottomtitlebar
\vskip 0.3in minus 0.1in}
\vspace{-0.3in}

\section{Geometric interpretation of probabilities $\bm{p_j^{(n)}}$ and $\bm{p_{j|i}^{(n)}}$}

\begin{figure*}[htp!]
\centering
\begin{minipage}{0.27\textwidth}
\includegraphics[width=\textwidth]{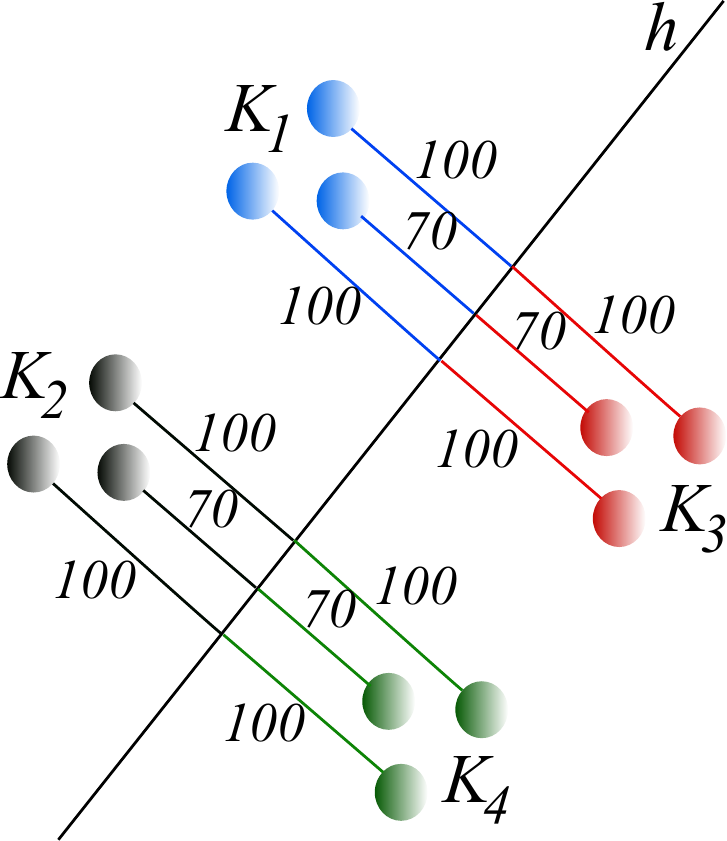}
\end{minipage}
\hspace{0.2in}\begin{minipage}{0.69\textwidth}
Discrete:\\
$p_1^{(n)} = \frac{6}{12} = \bm{0.5}$\\
$p_{1|1}^{(n)} = \frac{3}{3} = \bm{1}$, \hspace{0.2in}$p_{1|2}^{(n)} = \frac{3}{3} = \bm{1}$, \hspace{0.2in}$p_{1|3}^{(n)} = \frac{0}{3} = \bm{0}$, \hspace{0.2in}$p_{1|4}^{(n)} = \frac{0}{3} = \bm{0}$\\
\\
Continuous:\\
$p_1^{(n)} = \frac{1}{12}(\sigma(100) + \sigma(70) + \ldots + \sigma(-70) + \sigma(-100)) \approx \bm{0.5}$\\
$p_{1|1}^{(n)} = \frac{1}{3}(\sigma(100) + \sigma(70) + \sigma(100)) \approx \bm{1}$\\
$p_{1|2}^{(n)} = \frac{1}{3}(\sigma(100) + \sigma(70) + \sigma(100)) \approx \bm{1}$\\
$p_{1|3}^{(n)} = \frac{1}{3}(\sigma(-100) + \sigma(-70) + \sigma(-100)) \approx \bm{0}$\\
$p_{1|4}^{(n)} = \frac{1}{3}(\sigma(-100) + \sigma(-70) + \sigma(-100)) \approx \bm{0}$
\end{minipage}
\vspace{-0.1in}
\caption{The comparison of discrete and continuous definitions of probabilities $p_j^{(n)}$ and $p_{j|i}^{(n)}$ on a simple example with $\nlabels = 4$ classes and binary tree ($\arity = 2$). $n$ is an exemplary node, e.g. root. $\sigma$ denotes sigmoid function. Color circles denote data points.
}
\label{fig:relaxation}
\end{figure*}

\begin{remark}
One could define $p_j^{(n)}$ as the ratio of the number of examples that reach node $n$ and are sent to its $j^{\text{th}}$ child to the total the number of examples that reach node $n$ and $p_{j|i}^{(n)}$ as the ratio of the number of examples that reach node $n$, correspond to label $i$, and are sent to the $j^{\text{th}}$ child of node $n$ to the total the number of examples that reach node $n$ and correspond to label $i$. We instead look at the continuous counter-parts of these discrete definitions as given by Equations~\ref{eq:balprob} and~\ref{eq:pureprob} and illustrated in Figure~\ref{fig:relaxation} (note that continuous definitions have elegant geometric interpretation based on margins), which simplifies the optimization problem. 
\end{remark}

\section{Theoretical proofs}
\label{sec:proofs}

\begin{proof}[Proof of Lemma~\ref{lem:subgrad}]
Recall the form of the objective defined in \ref{eq:objective}:
\begin{eqnarray*}
J_n &=& \frac{2}{\arity} \sum_{i=1}^{\nlabels} \propor^{(n)}_i \Big( \sum_{j=1}^{\arity} |p^{(n)}_j - p^{(n)}_{j|i}| \Big)\\
&=& \frac{2}{\arity} \expct_{i \sim \propor^{(n)}} \Big[ f^J_n(i, p^{(n)}_{\cdot|\cdot}, \propor^{(n)})\Big]
\end{eqnarray*}
Where:
\begin{eqnarray*}
f^J_n(i, p^{(n)}_{\cdot|\cdot}, \propor^{(n)}) &=&  \sum_{j=1}^{\arity} \Big| p^{(n)}_j - p^{(n)}_{j|i} \Big| = \sum_{j=1}^\arity \Big| p^{(n)}_{j|i} - \sum_{i'=1}^\nlabels  \propor^{(n)}_{i'} p^{(n)}_{j|i'} \Big| \\ 
 &=& \sum_{j=1}^\arity \Big| \sum_{i'=1}^\nlabels   (\indic_{i = i'} - \propor^{(n)}_{i'}) p^{(n)}_{j|i'} \Big|
\end{eqnarray*}
Hence:
\begin{equation*}
\frac{\partial f^J_n(i, p^{(n)}_{\cdot|\cdot}, \propor^{(n)})}{ \partial p^{(n)}_{j|i}} = (1 - \propor^{(n)}_i) \sign(p^{(n)}_{j|i} - p^{(n)}_j)
\end{equation*}
And:
\begin{eqnarray*}
\frac{\partial f^J_n(i, p^{(n)}_{\cdot|\cdot}, \propor^{(n)})}{ \partial \log p^{(n)}_{j|i}} &=& (1 - \propor^{(n)}_i) \sign(p^{(n)}_{j|i} - p^{(n)}_j) \frac{\partial p^{(n)}_{j|i}}{ \partial \log p^{(n)}_{j|i}} \\
&=& (1 - \propor^{(n)}_i) \sign(p^{(n)}_{j|i} - p^{(n)}_j) p^{(n)}_{j|i}
\end{eqnarray*}
By assigning each label $j$ to a specific child $i$ under the constraint that no child has more than $L$ labels, we take a step in the direction $\partial E \in \{0, 1\}^{\arity \times \nlabels}$, where:
\begin{eqnarray*}
\forall i \in [1, \nlabels],& \sum_{j=1}^{\arity} \partial E_{j, i} = 1\\
&\text{and} \nonumber \\
\forall j \in [1, \arity],& \sum_{i=1}^{\nlabels} \partial E_{j, i} \leq L
\end{eqnarray*}
Thus:
\begin{eqnarray}
\frac{\partial J_n}{\partial p_{\cdot|\cdot}^{(n)}} \partial E &=& \frac{2}{M} \frac{\expct_{i \sim \propor^{(n)}} \Big[ f^J_n(i, p^{(n)}_{\cdot|\cdot}, \propor^{(n)})\Big]}{\partial p_{\cdot|\cdot}^{(n)}} \partial E \nonumber \\
&=& \frac{2}{M} \sum_{i=1}^{\nlabels}\propor_{i}^{(n)}(1 - \propor_{i}^{(n)}) \sum_{j=1}^M \Big( \sign(p^{(n)}_{j|i} - p^{(n)}_j) \partial E_{j,i} \Big) \label{eq:grad_tot}
\end{eqnarray}
And:
\begin{equation}
\frac{\partial J_n}{\partial \log p_{\cdot|\cdot}^{(n)}} \partial E = \frac{2}{M} \sum_{i=1}^{\nlabels}\propor_{i}^{(n)}(1 - \propor_{i}^{(n)}) \sum_{j=1}^M \Big( \sign(p^{(n)}_{j|i} - p^{(n)}_j) p^{(n)}_{j|i} \partial E_{j,i} \Big) \label{eq:grad_tot_log}
\end{equation}

If there exists such an assignment for which \ref{eq:grad_tot} is positive, then the greedy method proposed in \ref{algo:assign} finds it. Indeed, suppose that Algorithm \ref{algo:assign} assigns label $i$ to child $j$ and $i'$ to $j'$. Suppose now that another assignment $\partial E'$ sends $i$ to $j'$ and $i$ to $j'$. Then:
\begin{equation}
\frac{\partial J_n}{\partial p_{\cdot|\cdot}^{(n)}} \Big( \partial E - \partial E' \Big) = \Big( \frac{\partial J_n}{\partial p_{j|i}^{(n)}} + \frac{\partial J_n}{\partial p_{j'|i'}^{(n)}} \Big) - \Big( \frac{\partial J_n}{\partial p_{j|i'}^{(n)}} + \frac{\partial J_n}{\partial p_{j'|i}^{(n)}} \Big)
\label{eq:grad_diff}
\end{equation}
Since the algorithm assigns children by descending order of $\frac{\partial J_n}{\partial p_{j|i}^{(n)}}$ until a child $j$ is full, we have:
$$\frac{\partial J_n}{\partial p_{j|i}^{(n)}} \geq \frac{\partial J_n}{\partial p_{j|i'}^{(n)}} \;\;\;\;\;\; \text{and} \;\;\;\;\;\; \frac{\partial J_n}{\partial p_{'j|i'}^{(n)}} \geq \frac{\partial J_n}{\partial p_{j'|i}^{(n)}}$$
Hence:
$$\frac{\partial J_n}{\partial p_{\cdot|\cdot}^{(n)}} \Big( \partial E - \partial E' \Big) \geq 0$$
Thus, the greedy algorithm finds the assignment that most increases $J_n$ most under the children size constraints.

Moreover, $\frac{\partial J_n}{\partial p_{\cdot|\cdot}^{(n)}}$ is always positive for $L \leq M$ or $L \geq 2M(M-2)$.
\end{proof}

\begin{proof}[Proof of Lemma~\ref{lem:avgscorerange}]

Both $J_n$ and $J_T$ are defined as the sum of non-negative values which gives the lower-bound. We next derive the upper-bound on $J_n$. Recall:
\begin{eqnarray*}
J_n = \frac{2}{\arity}\sum_{j=1}^\arity\sum_{i=1}^\nlabels \propor^{(n)}_i|p_j^{(n)} - p^{(n)}_{j|i}| =  \frac{2}{\arity}\sum_{j=1}^\arity\sum_{i=1}^\nlabels \propor^{(n)}_i\left|\sum_{l=1}^\nlabels \propor^{(n)}_lp^{(n)}_{j|l} - p^{(n)}_{j|i}\right|
\end{eqnarray*}
since $p_j^{(n)} = \sum_{l=1}^\nlabels \propor^{(n)}_lp^{(n)}_{j|l}$. The objective $J_n$ is maximized on the extremes of the $[0,1]$ interval. Thus, define the following two sets of indices:
\[O_j = \{i: i\in\{1,2,\dots,\nlabels \}, p^{(n)}_{j|i} = 1\} \text{\:\:\:\:\:\:\:and\:\:\:\:\:\:\:} Z_j = \{i: i\in\{1,2,\dots,\nlabels\}, p^{(n)}_{j|i} = 0\}.
\]
We omit indexing these sets with $n$ for the ease of notation. We continue as follows
\begin{eqnarray*}
J_n &\leq& \frac{2}{\arity}\sum_{j=1}^\arity\left[\sum_{i \in O_j}\propor^{(n)}_i\left(1 - \sum_{l\in O_j}\propor^{(n)}_l\right) + \sum_{i \in Z_j}\propor^{(n)}_i\sum_{l\in O_j}\propor^{(n)}_l\right] \\
&=& \frac{4}{\arity}\sum_{j=1}^\arity\left[\sum_{i \in O_j}\propor^{(n)}_i - \left(\sum_{i \in O_j}\propor^{(n)}_i\right)^2\right] \\
&=& \frac{4}{\arity}\left[1 - \sum_{j=1}^\arity\left(\sum_{i \in O_j}\propor^{(n)}_i\right)^2\right], 
\end{eqnarray*}
where the last inequality is the consequence of the following: $\sum_{j=1}^\arity p^{(n)}_j = 1$ and $p^{(n)}_j = \sum_{l=1}^\nlabels \propor_l^{(n)}p^{(n)}_{j|l} = \sum_{i \in O_j}\propor^{(n)}_i$, thus $\sum_{j=1}^\arity\sum_{i \in O_j}\propor^{(n)}_i = 1$. Apllying Jensen's ineqality to the last inequality obtained gives
\begin{eqnarray*}
J_n &\leq& \frac{4}{\arity} - 4\left[\sum_{j=1}^\arity\left(\frac{1}{\arity}\sum_{i \in O_j}\propor^{(n)}_i\right)\right]^2 \\ 
&=& \frac{4}{\arity}\left(1 - \frac{1}{\arity}\right)
\end{eqnarray*}
That ends the proof. 
\end{proof}

\begin{proof}[Proof of Lemma~\ref{lem:avgscoremaxvalue}]
We start from proving that if the split in node $n$ is perfectly balanced, i.e. $\forall_{j = \{1,2,\dots,\arity\}}p^{(n)}_j = \frac{1}{\arity}$, and perfectly pure, i.e. $\forall_{\substack{ j = \{1,2,\dots,\arity\} \\ i = \{1,2,\dots,\nlabels\}}}\min(p^{(n)}_{j|i},1-p^{(n)}_{j|i}) = 0$, then $J_n$  admits the highest value $J_n = \frac{4}{\arity}\left(1 - \frac{1}{\arity}\right)$. Since the split is maximally balanced we write:
\[
J_n = \frac{2}{\arity}\sum_{j=1}^\arity\sum_{i=1}^\nlabels \propor^{(n)}_i\left|\frac{1}{\arity} - p^{(n)}_{j|i}\right|.
\]
Since the split is maximally pure, each $p^{(n)}_{j|i}$ can only take value $0$ or $1$. As in the proof of previous lemma, define two sets of indices:
\[O_j = \{i: i\in\{1,2,\dots,\nlabels\}, p^{(n)}_{j|i} = 1\} \text{\:\:\:\:\:\:\:and\:\:\:\:\:\:\:} Z_j = \{i: i\in\{1,2,\dots,\nlabels\}, p^{(n)}_{j|i} = 0\}.
\]
We omit indexing these sets with $n$ for the ease of notation. Thus
\begin{eqnarray*}
J_n &=& \frac{2}{\arity}\sum_{j=1}^\arity\left[\sum_{i\in O_j}\propor^{(n)}_i\left(1 - \frac{1}{\arity}\right) + \sum_{i\in Z_j}\propor^{(n)}_i\frac{1}{\arity}\right] \\
&=& \frac{2}{\arity}\sum_{j=1}^\arity\left[\sum_{i\in O_j}\propor^{(n)}_i\left(1 - \frac{1}{\arity}\right) + \frac{1}{\arity}\left(1-\sum_{i\in O_j}\propor^{(n)}_i\right)\right] \\
&=& \frac{2}{\arity}\left(1 - \frac{2}{\arity}\right)\sum_{j=1}^\arity\sum_{i\in O_j}\propor^{(n)}_i + \frac{2}{\arity} \\ 
&=& \frac{4}{\arity}\left(1 - \frac{1}{\arity}\right),
\end{eqnarray*}
where the last equality comes from the fact that $\sum_{j=1}^\arity p^{(n)}_j = 1$ and $p^{(n)}_j = \sum_{l=1}^\nlabels \propor^{(n)}_lp^{(n)}_{j|l} = \sum_{i \in O_j}\propor^{(n)}_i$, thus $\sum_{j=1}^\arity\sum_{i \in O_j}\propor^{(n)}_i = 1$.

Thus we are done with proving one induction direction. Next we prove that  if $J_n$  admits the highest value $J_n = \frac{4}{\arity}\left(1 - \frac{1}{\arity}\right)$, then the split in node $n$ is perfectly balanced, i.e. $\forall_{j = \{1,2,\dots,\arity\}}p^{(n)}_j = \frac{1}{\arity}$, and perfectly pure, i.e. $\forall_{\substack{ j = \{1,2,\dots,\arity\} \\ i = \{1,2,\dots,\nlabels\}}}\min(p^{(n)}_{j|i},1-p^{(n)}_{j|i}) = 0$.

Without loss of generality assume each $\propor^{(n)}_i \in (0,1)$. The objective $J_n$ is certainly maximized in the extremes of the interval $[0,1]$, where each $p^{(n)}_{j|i}$ is either $0$ or $1$. Also, at maximum it cannot be that for any given $j$, all $p^{(n)}_{j|i}$'s are $0$ or all $p^{(n)}_{j|i}$'s are $1$. The function $J(h)$ is differentiable in these extremes. Next, define three sets of indices:
\[\mathcal{A}_j = \{i:\sum_{l=1}^\nlabels \propor^{(n)}_ip^{(n)}_{j|l} \geq p^{(n)}_{j|i}\} \text{\:\:\:\:\:\:\:and\:\:\:\:\:\:\:} \mathcal{B}_j = \{i:\sum_{l=1}^\nlabels \propor^{(n)}_ip^{(n)}_{j|l} < p^{(n)}_{j|i}\} \text{\:\:\:\:\:\:\:and\:\:\:\:\:\:\:} \mathcal{C}_j = \{i:\sum_{l=1}^\nlabels \propor^{(n)}_ip^{(n)}_{j|l} > p^{(n)}_{j|i}\}.
\]
We omit indexing these sets with $n$ for the ease of notation. Objective $J_n$ can then be re-written as
\[J_n = \frac{2}{\arity}\sum_{j=1}^\arity\left[\sum_{i\in \mathcal{A}_j} \propor^{(n)}_i \left(\sum_{l=1}^\nlabels \propor^{(n)}_ip^{(n)}_{j|l} - p^{(n)}_{j|i}\right) + 2\sum_{i\in \mathcal{B}_j} \propor^{(n)}_i \left( p^{(n)}_{j|i} - \sum_{l=1}^\nlabels \propor^{(n)}_ip^{(n)}_{j|l}\right)\right], 
\]
We next compute the derivatives of $J_n$ with respect to $p^{(n)}_{j|z}$, where $z = \{1,2,\dots,\nlabels\}$, everywhere where the function is differentiable and obtain
\[\frac{\partial J_n}{\partial p^{(n)}_{j|z}} = \left \{
  \begin{tabular}{c}
 $2\propor^{(n)}_z(\sum_{i\in\mathcal{C}_j}\propor^{(n)}_i - 1)\:\:\:\:\:$ if$\:$$z\in\mathcal{C}_j$\\
  $\:2\propor^{(n)}_z(1 - \sum_{i\in\mathcal{B}_j}\propor^{(n)}_i)\:\:\:\:$ if$\:$$z\in\mathcal{B}_j$
  \end{tabular}
\right.,
\]
Note that in the extremes of the interval $[0,1]$ where $J_n$ is maximized, it cannot be that $\sum_{i\in\mathcal{C}_j}\propor^{(n)}_i = 1$ or $\sum_{i\in\mathcal{B}_j}\propor^{(n)}_i = 1$ thus the gradient is non-zero. This fact and the fact that $J_n$ is convex imply that $J_n$ can \textit{only} be maximized at the extremes of the $[0,1]$ interval. Thus if $J_n$ admits the highest value, then the node split is perfectly pure. We still need to show that if $J_n$ admits the highest value, then the node split is also perfectly balanced. We give a proof by contradiction, thus we assume that at least for one value of $j$, $p^{(n)}_j \neq \frac{1}{\arity}$, or in other words if we decompose each $p^{(n)}_j$ as $p^{(n)}_j = \frac{1}{\arity} + x_j$, then at least for one value of $j$, $x_j \neq 0$. Lets once again define two sets of indices (we omit indexing $x_j$ and these sets with $n$ for the ease of notation):
\[O_j = \{i: i\in\{1,2,\dots,\nlabels\}, p^{(n)}_{j|i} = 1\} \text{\:\:\:\:\:\:\:and\:\:\:\:\:\:\:} Z_j = \{i: i\in\{1,2,\dots,\nlabels\}, p^{(n)}_{j|i} = 0\},
\]
and recall that $p^{(n)}_j = \sum_{l=1}^\nlabels \propor^{(n)}_lp^{(n)}_{j|l} = \sum_{i \in O_j}\propor^{(n)}_i$.
We proceed as follows
\begin{eqnarray*}
\frac{4}{\arity}\left(1 - \frac{1}{\arity}\right) = J_n &=& \frac{2}{\arity}\sum_{j=1}^\arity\left[\sum_{i \in O_j} \propor^{(n)}_i(1 - p^{(n)}_j) + \sum_{i \in Z_j} \propor^{(n)}_ip^{(n)}_j\right] \\
&=& \frac{2}{\arity}\sum_{j=1}^\arity\left[p^{(n)}_j(1 - p^{(n)}_j) + p^{(n)}_j(1-p^{(n)}_j)\right] \\
&=& \frac{4}{\arity}\sum_{j=1}^\arity\left[p^{(n)}_j - (p^{(n)}_j)^2\right] \\
&=& \frac{4}{\arity}\left[1 - \sum_{j=1}^\arity(p^{(n)}_j)^2\right] \\
&=& \frac{4}{\arity}\left[1 - \sum_{j=1}^\arity\left(\frac{1}{\arity} + x_j\right)^2\right] \\
&=& \frac{4}{\arity}\left(1 - \frac{1}{\arity} - \frac{2}{\arity}\sum_{j=1}^\arity x_j - \sum_{j=1}^\arity x_j^2\right) \\
&<& \frac{4}{\arity}\left(1 - \frac{1}{\arity}\right) \\
\end{eqnarray*}
Thus we obtain the contradiction which ends the proof.
\end{proof}

\begin{proof}[Proof of Lemma~\ref{lem:isol_bal}]
Since we node that the split is perfectly pure, then each $p_{j|i}^{(n)}$ is either $0$ or $1$. Thus we define two sets
\[O_j = \{i: i\in\{1,2,\dots,\nlabels\}, p^{(n)}_{j|i} = 1\} \text{\:\:\:\:\:\:\:and\:\:\:\:\:\:\:} Z_j = \{i: i\in\{1,2,\dots,\nlabels\}, p^{(n)}_{j|i} = 0\}.
\]
and thus
\[J_n = \frac{2}{\arity}\sum_{j=1}^\arity\left[\sum_{i\in O_j}\propor^{(n)}_i\left(1 - p_j\right) + \sum_{i\in Z_j}\propor^{(n)}_ip_j\right] 
\]
Note that $p_j = \sum_{i\in O_j}\propor_i^{(n)}$.
Then
\[J_n = \frac{2}{\arity}\sum_{j=1}^\arity\left[p_j\left(1 - p_j\right) + (1-p_j)p_j\right] = \frac{4}{\arity}\sum_{j=1}^\arity p_j\left(1 - p_j\right) = \frac{4}{\arity}\left(1 - \sum_{j=1}^{\arity}p_j^2\right)
\]
and thus 
\begin{equation}
\sum_{j=1}^{\arity}p_j^2 = 1 - \frac{MJ_n}{4}.
\label{eq:onehelp}
\end{equation}
Lets express $p_j$ as $p_j = \frac{1}{\arity} + \epsilon_j$, where $\epsilon_j \in [-\frac{1}{\arity}, 1 - \frac{1}{\arity}]$. Then
\begin{equation}
\sum_{j=1}^{\arity}p_j^2 = \sum_{j=1}^{\arity}\left(\frac{1}{\arity} + \epsilon_j\right)^2 = \frac{1}{\arity} + \frac{2}{\arity}\sum_{j=1}^{\arity}\epsilon_j + \sum_{j=1}^{\arity}\epsilon_j^2 = \frac{1}{\arity} + \sum_{j=1}^{\arity}\epsilon_j^2,
\label{eq:twohelp}
\end{equation}
since $\frac{2}{\arity}\sum_{j=1}^{\arity}\epsilon_j = 0$.
Thus combining Equation~\ref{eq:onehelp} and~\ref{eq:twohelp}
\[\frac{1}{\arity} + \sum_{j=1}^{\arity}\epsilon_j^2 = 1 - \frac{MJ_n}{4}
\]
and thus
\[\sum_{j=1}^{\arity}\epsilon_j^2 = 1 - \frac{1}{\arity} - \frac{MJ_n}{4}.
\]
The last statement implies that
\[\max_{j = 1,2,\dots,\arity}\epsilon_j \leq \sqrt{1 - \frac{1}{\arity} - \frac{MJ_n}{4}},
\]
which is equivalent to
\[\min_{j = 1,2,\dots,\arity} p_j = \frac{1}{\arity} - \max_j\epsilon_j \geq \frac{1}{\arity} - \sqrt{1 - \frac{1}{\arity} - \frac{MJ_n}{4}} = \frac{1}{\arity} - \frac{\sqrt{M(J^{*} - J_n)}}{2}.
\]
\end{proof}

\begin{proof}[Proof of Lemma~\ref{lem:isol_pur}]
Since the split is perfectly balanced we have the following:
\[J_n = \frac{2}{\arity}\sum_{j=1}^{\arity}\sum_{i=1}^\nlabels \propor_i^{(n)}\left|\frac{1}{\arity} - p_{j|i}^{(n)}\right| = \frac{2}{\arity}\sum_{i=1}^\nlabels\sum_{j=1}^{\arity} \propor_i^{(n)}\left|\frac{1}{\arity} - p_{j|i}^{(n)}\right|
\]
Define two sets
\[\mathcal{A}_i = \{j: j\in\{1,2,\dots,\nlabels\}, p^{(n)}_{j|i} < \frac{1}{\arity}\} \text{\:\:\:\:\:\:\:and\:\:\:\:\:\:\:} \mathcal{B}_i = \{j: j\in\{1,2,\dots,\nlabels\}, p^{(n)}_{j|i} \geq \frac{1}{\arity}\}.
\]
Then
\begin{eqnarray*}
J_n &=& \frac{2} {\arity}\sum_{i=1}^\nlabels \left[ \sum_{j\in\mathcal{A}_i}  \propor_i^{(n)}\left(\frac{1}{\arity} - p_{j|i}^{(n)}\right) + \sum_{j\in\mathcal{B}_i} \propor_i^{(n)}\left(p_{j|i}^{(n)} - \frac{1}{\arity}\right) \right]\\
&=& \frac{2} {\arity}\sum_{i=1}^\nlabels \propor_i^{(n)} \left[\sum_{j\in\mathcal{A}_i} \left(\frac{1}{\arity} - p_{j|i}^{(n)}\right) + \sum_{j\in\mathcal{B}_i} \left(p_{j|i}^{(n)} - \frac{1}{\arity}\right) \right] \\
&=& \frac{2} {\arity}\sum_{i=1}^\nlabels \propor_i^{(n)} \left[\sum_{j\in\mathcal{A}_i} \left(\frac{1}{\arity} - p_{j|i}^{(n)}\right) + \sum_{j\in\mathcal{B}_i} \left( (1 - \frac{1}{\arity}) - (1 - p_{j|i}^{(n)}) \right) \right] \\
\end{eqnarray*}
Recall that the optimal value of $J_n$ is:
\[J^* = \frac{4}{\arity}\left(1 - \frac{1}{\arity} \right) = \frac{2}{\arity}\sum_{i=1}^{\nnodes}\propor_i^{(n)} \left[ \left(\arity - 1 \right)\frac{1}{\arity} + \left( 1 - \frac{1}{\arity} \right) \right] = \frac{2}{\arity}\sum_{i=1}^{\nnodes}\propor_i^{(n)} \left[ \left( \sum_{j \in \mathcal{A}_i \cup  \mathcal{B}_i} \frac{1}{\arity} \right) - \frac{1}{\arity} + \left( 1 - \frac{1}{\arity} \right) \right] \]
Note $\mathcal{A}_i$ can have at most $\arity - 1$ elements. Furthermore, $\forall j \in \mathcal{A}_i, p_{j|i}^{(n)} < 1 - p_{j|i}^{(n)}$. Then, we have:
\begin{equation*}
J^* - J^n = \frac{2} {\arity}\sum_{i=1}^\nlabels \propor_i^{(n)} \left[ \sum_{j\in\mathcal{A}_i}  p_{j|i}^{(n)} + \sum_{j\in\mathcal{B}_i} \left( (1 - p_{j|i}^{(n)}) + \frac{1}{\arity} - (1 - \frac{1}{\arity}) \right) - \frac{1}{\arity} + \left( 1 - \frac{1}{\arity} \right) \right]
\end{equation*}
Hence, since $\mathcal{B}_i$ has at least one element:
\begin{eqnarray*}
J^* - J^n &\geq& \frac{2} {\arity}\sum_{i=1}^\nlabels \propor_i^{(n)} \left[ \sum_{j\in\mathcal{A}_i}  p_{j|i}^{(n)} + \sum_{j\in\mathcal{B}_i} \left( 1 - p_{j|i}^{(n)} \right) \right]\\
&\geq& \frac{2} {\arity}\sum_{i=1}^\nlabels \propor_i^{(n)} \left[ \sum_{j = 1}^{\arity}\min(p_{j|i}^{(n)}, 1 - p_{j|i}^{(n)}) \right]\\
&\geq&  2 \alpha
\end{eqnarray*}
\end{proof}

\begin{proof}[Proof of Theorem~\ref{thm:errorboundgen}]
Let the weight of the tree leaf be defined as the probability that a randomly chosen data point $x$ drawn from some fixed target distribution $\mathcal{P}$ reaches this leaf. Suppose at time step $t$, $n$ is the heaviest leaf and has weight $w$. Consider splitting this leaf to $\arity$ children $n_1, n_2, \dots, n_\arity$. Let the weight of the $j^{\text{th}}$ child be denoted as $w_j$. Also for the ease of notation let $p_j$ refer to $p_j^{(n)}$ (recall that $\sum_{j=1}^mp_j = 1$)
and $p_{j|i}$ refer to $p_{j|i}^{(n)}$, and furthermore let $\propor_i$ be the shorthand for
$\propor^{(n)}_i$. Recall that $p_j = \sum_{i=1}^\nlabels \propor_ip_{j|i}$ and
$\sum_{i=1}^\nlabels \propor_i = 1$. Notice that for any $j = \{1,2,\dots,\arity\}$, $w_j = wp_j$. Let ${\bm \propor}$ be the $k$-element vector with $i^{th}$
entry equal to $\propor_i$. Define the following function: $\tilde{G}^e({\bm \propor}) = \sum_{i =
  1}^\nlabels \propor_{i}\ln \left( \frac{1}{\propor_{i}} \right)$. Recall the expression for the entropy of tree leaves: $G^e = \sum_{l \in \mathcal{L}}w_l\sum_{i=1}^\nlabels \propor^{(l)}_{i}\ln \left( \frac{1}{\propor^{(l)}_{i}} \right)$, where $\mathcal{L}$ is a set of all tree leaves. Before the split the contribution of node $n$ to 
$G^e$ was equal to $w\tilde{G}^e({\bm \propor})$. Note that for any $j = \{1,2,\dots,\arity\}$, $\propor^{(n_j)}_i =
\frac{\propor_ip_{j|i}}{p_j}$ is the probability that a randomly chosen $x$
drawn from $\mathcal{P}$ has label $i$ given that $x$ reaches node
$n_j$. For brevity, let $\propor^{n^j}_i$ be denoted as $\propor_{j,i}$. Let ${\bm \propor}_j$ be the
$k$-element vector with $i^{th}$ entry equal to $\propor_{j,i}$. Notice that ${\bm \propor} = \sum_{j=1}^\arity p_j{\bm \propor}_j$. After the split the contribution of the same,
now internal, node $n$ changes to $w\sum_{j=1}^\arity p_j\tilde{G}^e({\bm \propor}_j)$. We denote the difference between the contribution of node $n$ to the value of the entropy-based objectives in times $t$ and $t+1$ as
\begin{equation}
  \Delta_t^e := G_t^e - G_{t+1}^e =  w\left[\tilde{G}^e({\bm \propor}) - \sum_{j=1}^\arity p_j\tilde{G}^e({\bm \propor}_j)\right].
  \label{eqn:ent-decrease}
\end{equation}
The entropy function $\tilde{G}^e$ is strongly concave with respect to $l_1$-norm with modulus $1$, thus we extend the inequality given by Equation 7 in~\cite{AKM2016} by applying Theorem 5.2. from~\cite{AGNS2011} and obtain the following bound
\begin{eqnarray*}
  \Delta_t^e &=&  w\left[\tilde{G}^e({\bm \propor}) - \sum_{j=1}^\arity p_j\tilde{G}^e({\bm \propor}_j)\right]\\
  &\geq&  w\frac{1}{2}\sum_{j=1}^\arity p_j\|\propor_j - \sum_{l=1}^\arity p_l\propor_l\|_1^2\\
  &=&  w\frac{1}{2}\sum_{j=1}^\arity p_j\left( \sum_{i=1}^\nlabels \left| \frac{\propor_i p_{j|i}}{p_j} - \sum_{l=1}^\arity p_l\frac{\propor_i p_{l|i}}{p_l} \right| \right)^2 \\
    &=&  w\frac{1}{2}\sum_{j=1}^\arity p_j\left( \sum_{i=1}^\nlabels \propor_i\left| \frac{p_{j|i}}{p_j} - \sum_{l=1}^\arity p_{l|i} \right| \right)^2 \\
        &=&  w\frac{1}{2}\sum_{j=1}^\arity p_j\left( \sum_{i=1}^\nlabels \propor_i\left| \frac{p_{j|i}}{p_j} - 1 \right| \right)^2 \\
                &=&  w\frac{1}{2}\sum_{j=1}^\arity \frac{1}{p_j}\left( \sum_{i=1}^\nlabels \propor_i\left| p_{j|i} - p_j \right| \right)^2.
\end{eqnarray*}
Before proceeding, we will bound each $p_j$. Note that by the \textit{Weak Hypothesis Assumption} we have 
\[\gamma \in \left[\frac{\arity}{2}\min_{j=1,2,\dots,\arity}p_j, 1 - \frac{\arity}{2}\min_{j=1,2,\dots,\arity}p_j\right],
\]
thus
\[\min_{j=1,2,\dots,\arity}p_j \geq \frac{2\gamma}{\arity},
\]
thus
all $p_j$s are such that $p_j\geq \frac{2\gamma}{\arity}$. Thus
\[\max_{j=1,2,\dots,\arity}p_j \leq 1 - \frac{2\gamma}{\arity}(\arity-1) = \frac{\arity(1-2\gamma)+2\gamma}{\arity}.
\]
Thus all $p_j$s are such that $p_j \leq \frac{\arity(1-2\gamma)+2\gamma}{\arity}$.
\begin{eqnarray*}
\Delta_t^e &\geq&  w\frac{\arity^2}{2[(\arity(1-2\gamma)+2\gamma]}\sum_{j=1}^\arity \frac{1}{\arity}\left( \sum_{i=1}^\nlabels \propor_i\left| p_{j|i} - p_j \right| \right)^2 \\
                &\geq&  w\frac{\arity^2}{2[(\arity(1-2\gamma)+2\gamma]}\left(\sum_{j=1}^\arity \frac{1}{\arity}\sum_{i=1}^\nlabels \propor_i\left| p_{j|i} - p_j \right| \right)^2 \\
                &=&  w\frac{\arity^2}{8[(\arity(1-2\gamma)+2\gamma]}\left(\frac{2}{\arity}\sum_{j=1}^\arity\sum_{i=1}^\nlabels \propor_i\left| p_{j|i} - p_j \right| \right)^2 \\
         &=&  \frac{\arity^2}{[(\arity(1-2\gamma)+2\gamma]}\frac{wJ_n^2}{8},
\end{eqnarray*}
where the last inequality is a consequence of Jensen's inequality. $w$ can further be lower-bounded by noticing the following
\[G^e_t = \sum_{l \in \mathcal{L}}w_l\sum_{i=1}^\nlabels \propor^{(l)}_{i}\ln \left( \frac{1}{\propor^{(l)}_{i}} \right) \leq \sum_{l \in \mathcal{L}}w_l\ln \nlabels \leq w\ln \nlabels \sum_{l \in \mathcal{L}}1 = [t(\arity-1)+1]w\ln \nlabels \leq (t+1)(\arity-1)w\ln \nlabels,
\]
where the first inequality results from the fact that uniform distribution maximizes the entropy. 

This gives the lower-bound on $\Delta_t^e$ of the following form:
\[\Delta_t^e \geq \frac{\arity^2G^e_tJ_n^2}{8(t+1)[\arity(1-2\gamma) + 2\gamma](\arity-1)\ln \nlabels},
\]
and by using \textit{Weak Hypothesis Assumption} we get
\[\Delta_t^e \geq \geq \frac{\arity^2G^e_t\gamma^2}{8(t+1)[\arity(1-2\gamma) + 2\gamma](\arity-1)\ln \nlabels}
\]
Following the recursion of the proof in Section 3.2 in~\cite{AKM2016} (note that in our case $G^e_1 \leq 2(\arity-1)\ln \nlabels$), we obtain that under the \textit{Weak Hypothesis Assumption}, for any $\kappa \in [0,2(\arity-1)\ln \nlabels]$, to obtain $G_t^e \leq \kappa$ it suffices to make
\[t \geq \left(\frac{2(\arity - 1)\ln \nlabels}{\kappa}\right)^{\frac{16[\arity(1-2\gamma) + 2\gamma](\arity-1)\ln \nlabels}{\arity^2\log_2 e\gamma^2}}
\]
splits. We next proceed to directly proving the error bound. Denote $w(l)$ to be the probability that a data point $x$ reached leaf $l$. Recall that $\propor^{(l)}_i$ is the probability that the data point $x$ corresponds to label $i$ given that $x$ reached $l$, i.e. $\propor^{(l)}_i = P(y(x) = i|x \:\:\text{reached}\:\: l)$. Let the label assigned to the leaf be the majority label and thus lets assume that the leaf is assigned to label $i$ if and only if the following is true $\forall_{\substack{z = \{1,2,\dots,k\} \\ z\neq i}}\propor^{(l)}_i\geq \propor^{(l)}_z$.
Therefore we can write that
\begin{eqnarray}
\!\!\!\!\!\!\!\!\epsilon(\mathcal{T}) \!\!\!\!&=&\!\!\!\! \sum_{i=1}^{\nlabels}P(t(x) = i, y(x) \neq i)\\
&=& \sum_{l \in \mathcal{L}} w(l)\sum_{i=1}^{\nlabels}P(t(x) = i, y(x) \neq i|x \:\:\text{reached}\:\:l) \nonumber\\
&=& \sum_{l \in \mathcal{L}} w(l)\sum_{i=1}^{\nlabels}P( y(x) \neq i|t(x) = i,x \:\:\text{reached}\:\:l)P(t(x) = i|x \:\:\text{reached}\:\:l) \nonumber\\
&=& \sum_{l \in \mathcal{L}} w(l)(1 - \max(\propor^{(l)}_1,\propor^{(l)}_2,\dots,\propor^{(l)}_\nlabels))\sum_{i=1}^{\nlabels}P(t(x) = i|x \:\:\text{reached}\:\:l) \nonumber\\
&=& \sum_{l \in \mathcal{L}} w(l)(1 - \max(\propor^{(l)}_1,\propor^{(l)}_2,\dots,\propor^{(l)}_\nlabels))
\end{eqnarray}

Consider again the Shannon entropy $G(\mathcal{T})$ of the leaves of tree $\mathcal{T}$ that is defined as 
\begin{equation}
G^e(\mathcal{T}) = \sum_{l \in \mathcal{L}} w(l)\sum_{i=1}^{\nlabels}\propor^{(l)}_i\log_2\frac{1}{\propor^{(l)}_i}.
\label{eq:entropy}
\end{equation}

Let $i_l = \arg\max_{i = \{1,2,\dots,\nlabels\}}\propor^{(l)}_i$.
Note that
\begin{eqnarray}
G^e(\mathcal{T}) &=& \sum_{l \in \mathcal{L}} w(l)\sum_{i=1}^{\nlabels}\propor^{(l)}_i\log_2\frac{1}{\propor^{(l)}_i} \nonumber\\
&\geq& \sum_{l \in \mathcal{L}} w(l)\sum_{\substack{i=1 \\ i \neq i_l}}^{\nlabels}\propor^{(l)}_i\log_2\frac{1}{\propor^{(l)}_i} \nonumber\\
&\geq& \sum_{l \in \mathcal{L}} w(l)\sum_{\substack{i=1 \\ i \neq i_l}}^{\nlabels}\propor^{(l)}_i \nonumber\\
&=& \sum_{l \in \mathcal{L}} w(l)(1 - \max(\propor^{(l)}_1,\propor^{(l)}_2,\dots,\propor^{(l)}_\nlabels)) \nonumber\\
&=& \epsilon(\mathcal{T}),
\end{eqnarray}
where the last inequality comes from the fact that $\forall_{\substack{i=\{1,2,\dots,\nlabels\} \\ i \neq i_l}}\propor^{(l)}_i\leq 0.5$ and thus $\forall_{\substack{i=\{1,2,\dots,\nlabels\} \\ i \neq i_l}}\frac{1}{\propor^{(l)}_i} \in [2;+\infty]$ and consequently $\forall_{\substack{i=\{1,2,\dots,\nlabels\} \\ i \neq i_l}}\log_2\frac{1}{\propor^{(l)}_i} \in [1;+\infty]$.

We next use the proof of Theorem 6 in~\cite{AKM2016}. The proof modifies only slightly for our purposes and thus we only list these modifications below.
\begin{itemize}
\item Since we define the Shannon entropy through logarithm with base $2$ instead of the natural logarithm, the right hand side of inequality (2.6) in~\cite{ShaiSS2012} should have an additional multiplicative factor equal to $\frac{1}{\ln 2}$ and thus the right-hand side of the inequality stated in Lemma 14 has to have the same multiplicative factor.
\item For the same reason as above, the right-hand side of the inequality in Lemma 9 should take logarithm with base $2$ of $k$ instead of the natural logarithm of $k$.
\end{itemize}
Propagating these changes in the proof of Theorem 6 results in the statement of Theorem~\ref{thm:errorboundgen}.

\end{proof}

\begin{proof}[Proof of Corollary~\ref{corr:errorboundgensquare}]
Note that the lower-bound on $\Delta_t^e$ from the previous prove could be made tighter as follows:
\begin{eqnarray*}
  \Delta_t^e &\geq& w\frac{1}{2}\sum_{j=1}^\arity \frac{1}{p_j}\left( \sum_{i=1}^\nlabels \propor_i\left| p_{j|i} - p_j \right| \right)^2 \\
                &=&  w\frac{\arity^2}{2}\sum_{j=1}^\arity \frac{1}{\arity}\left( \sum_{i=1}^\nlabels \propor_i\left| p_{j|i} - p_j \right| \right)^2 \\
                &\geq&  w\frac{\arity^2}{2}\left(\sum_{j=1}^\arity \frac{1}{\arity}\sum_{i=1}^\nlabels \propor_i\left| p_{j|i} - p_j \right| \right)^2 \\
                &=&  w\frac{\arity^2}{8}\left(\frac{2}{\arity}\sum_{j=1}^\arity\sum_{i=1}^\nlabels \propor_i\left| p_{j|i} - p_j \right| \right)^2 \\
         &=&  \frac{\arity^2wJ_n^2}{8},
\end{eqnarray*}
where the first inequality was taken from the proof of Theorem~\ref{thm:errorboundgen} and the following equality follows from the fact that each node is balanced. By next following exactly the same steps as shown in the proof of Theorem~\ref{thm:errorboundgen} we obtain the corollary.
\end{proof}

\section{Experimental Setting}

\subsection{Classification}

For the YFCC100M experiments, we learned our models with SGD with a linearly decreasing rate for five epochs. We run a hyper-parameter search on the learning rate (in $\{0.01, 0.02, 0.05, 0.1, 0.25, 0.5\}$). In the learned tree settings, the learning rate stays constant for the first half of training, during which the AssignLabels() routine is called 50 times. We run the experiments in a Hogwild data-parallel setting using 12 threads on an Intel Xeon E5-2690v4 2.6GHz CPU. At prediction time, we perform a truncated depth first search to find the most likely label (using the same idea as in a branch-and-bound algorithm: if a node score is less than that of the best current label, then all of its descendants are out).

\begin{table}[t!]
\setlength{\tabcolsep}{5.2pt}
\vspace{-0.025in}
\centering
\begin{tabular}{c l c c c c c}
\toprule
$d$ &\multicolumn{1}{c}{Model} & Arity & Prec & Rec & Train & Test \\ \midrule
\multirow{6}{*}{50} & TagSpace  & -    & 30.1 & -      & 3h8   & 6h\\ \cmidrule{2-7}
 & FastText  & 2    & 27.2 & 4.17   & \textbf{8m} & \textbf{1m}   \\ \cmidrule{2-7}
 & \multirow{2}{*}{Huffman Tree} & 5 & 28.3 & 4.33   & \textbf{8m}  & \textbf{1m}  \\
 &    & 20   & 29.9 & 4.58   & 10m   & 3m\\ \cmidrule{2-7}
 & \multirow{2}{*}{Learned Tree} & 5 & 31.6 & 4.85 & 18m   & \textbf{1m}   \\ 
 &         & 20   & \textbf{32.1} & \textbf{4.92}   & 30m   & 3m   \\ \midrule
\multirow{6}{*}{200} & TagSpace  & -    & 35.6 & - & 5h32  & 15h  \\ \cmidrule{2-7}
 & FastText  & 2    & 35.2 & 5.4    & \textbf{12m} & \textbf{1m}   \\ \cmidrule{2-7}
 & \multirow{2}{*}{Huffman Tree} & 5 & 35.8 & 5.5  & 13m   & 2m   \\
  & & 20   & 36.4 & 5.59   & 18m   & 3m   \\ \cmidrule{2-7}
 &\multirow{2}{*}{Learned Tree} & 5  & 36.1 & 5.53   & 35m   & 3m   \\
  &  & 20    & \textbf{36.6} & \textbf{5.61}   & 45m   & 8m   \\ \bottomrule
\end{tabular}
\caption{Classification performance on the YFCC100M dataset.}
\label{tab:class-res-full}
\end{table}

\subsection{Density Estimation}

In our experiments, we use a context window size of 4. We optimize the objectives with Adagrad, run a hyper-parameter search on the batch size (in $\{32, 64, 128\}$) and learning rate (in $\{0.01, 0.02, 0.05, 0.1, 0.25, 0.5\}$). The hidden representation dimension is $200$. In the learned tree settings, the AssignLabels() routine is called 50 times per epoch. We used a 12GB NVIDIA GeForce GTX TITAN GPU and all tree-based models are 65-ary.

For the Cluster Tree, we learn dimension 50 word embeddings with FastTree for 5 epochs using a hierarchical softmax loss, then obtain $45 = 65^2$ centroids using the ScikitLearn implementation of MiniBatchKmeans, and greedily assign words to clusters until full (when a cluster has 65 words).

\begin{algorithm}[h!]
\caption{Label Assignment Algorithm under Depth Constraint}
\label{algo:assign-depth}
\begin{multicols}{2}
\begin{tabular}{l}
\textbf{Input} Node statistics, max depth $D$\\
\hspace{0.33in} Paths from root to labels: $\mathcal{P} = (\mathbf{c}^i)_{i=1}^\nlabels$\\
\hspace{0.33in} node ID $n$ and depth $d$\\
\hspace{0.33in} List of labels currently reaching the node\\
\textbf{Ouput} Updated paths\\
\hspace{0.38in} Lists of labels now assigned to each of $n$'s\\
\hspace{0.38in} children under depth constraints\\
\\
\textbf{procedure} \textsf{AssignLabels} (labels, $n$, $d$)\\
\hspace{0.2in} // {\sl{first, compute $p_j^{(n)}$ and $p_{j|i}^{(n)}$. $\odot$ is the element-wise}}\\
\hspace{0.2in} // {\sl{multiplication}}\\
\hspace{0.2in} $\mathbf{p}^{avg}_0 \leftarrow \mathbf{0}$\\
\hspace{0.2in} $\text{count} \leftarrow 0$\\
\hspace{0.2in} \textbf{for} $i$ in labels \textbf{do}\\
\hspace{0.4in} $\mathbf{p}^{avg}_0 \leftarrow \mathbf{p}^{avg}_0 + \text{SumProbas}_{n, i}$\\
\hspace{0.4in} $\text{count} \leftarrow \text{count} + \text{Counts}_{n, i}$\\
\hspace{0.4in} $\mathbf{p}^{avg}_i \leftarrow \text{SumProbas}_{n, i} / \text{Counts}_{n, i}$\\
\hspace{0.2in} $\mathbf{p}^{avg}_0 \leftarrow \mathbf{p}^{avg}_0 / \text{count}$\\
\end{tabular}

\begin{tabular}{l}
\hspace{0.2in} // {\sl{then, assign each label to a child of $n$ under depth}}\\
\hspace{0.2in} // {\sl{constraints}}\\
\hspace{0.2in} unassigned $\leftarrow$ labels\\
\hspace{0.2in} full $\leftarrow \emptyset$\\
\hspace{0.2in} \textbf{for} $j=1$ to $\arity$ \textbf{do}\\
\hspace{0.4in} $\text{assigned}_j \leftarrow \emptyset$\\
\hspace{0.2in} \textbf{while} $\text{unassigned} \neq \emptyset$ \textbf{do}\\
\hspace{0.4in} \Big{/}\!\!\Big{/}{\sl{$\frac{\partial J_n}{ \partial p^{(n)}_{j|i}}$ is given in Equation \ref{eq:gradients}}}\\
\hspace{0.4in} $(i^*, j^*) \leftarrow \operatorname*{argmax}\limits_{i \in \text{unassigned}, j \not \in \text{full}}\left(\frac{\partial J_n}{ \partial p^{(n)}_{j|i}}\right)$ \label{algo:line:grad_sort}\\
\hspace{0.4in} $\mathbf{c}^{i^*}_d \leftarrow (n, j^*)$\\
\hspace{0.4in} $\text{assigned}_{j^*} \leftarrow \text{assigned}_{j^*} \cup \{ i^* \}$\\
\hspace{0.4in} $\text{unassigned} \leftarrow \text{unassigned} \setminus \{ i^* \}$\\
\hspace{0.4in} \textbf{if} $|\text{assigned}_{j^*}| = \arity^{D - d}$ \textbf{then}\\
\hspace{0.6in} $\text{full} \leftarrow \text{full} \cup \{j^*\}$\\
\hspace{0.2in} \textbf{for} $j=1$ to $\arity$ \textbf{do}\\
\hspace{0.4in} \textsf{AssignLabels} ($\text{assigned}_j$, $\text{child}_{n, j}$, $d + 1$)\\
\hspace{0.2in} \textbf{return} assigned
\end{tabular}

\end{multicols}
\end{algorithm}

\newpage
\begin{table}[htp!]
\centering
\begin{tabular}{|l|l|l|l|}
\hline
Leaf 229   & Leaf 230 & Leaf 300    & Leaf 231  \\  \hhline{=#=#=#=}
suggested  & vegas    & payments    & operates  \\ \hline
watched    & \&       & buy-outs    & includes  \\ \hline
created    & calif.   & swings      & intends   \\ \hline
violated   & park     & gains       & makes     \\ \hline
introduced & n.j.     & taxes       & means     \\ \hline
discovered & conn.    & operations  & helps     \\ \hline
carried    & pa.      & profits     & seeks     \\ \hline
described  & pa.      & penalties   & reduces   \\ \hline
accepted   & ii       & relations   & continues \\ \hline
listed     & d.       & liabilities & fails     \\ \hline
\ldots     & \ldots   & \ldots		& \ldots    \\ \hline
\end{tabular}
\caption{Example of labels reaching leaf nodes in the final tree. We can identify a leaf for 3rd person verbs, one for past participates, one for plural nouns, and one (loosely) for places.}
\label{tab:leaves}
\end{table}

\end{document}